\documentclass{article}

\usepackage{microtype}
\usepackage{graphicx}
\usepackage{tabularx}
\usepackage{booktabs} 

\usepackage{hyperref}

\usepackage[accepted]{icml2023}

\usepackage{amsmath}
\usepackage{amssymb}
\usepackage{mathtools}
\usepackage{amsthm}
\usepackage{algorithm, algorithmic}
\usepackage[utf8]{inputenc} 
\usepackage[T1]{fontenc}    
\usepackage{hyperref}       
\usepackage{url}            
\usepackage{booktabs}       
\usepackage{amsfonts}       
\usepackage{nicefrac}       
\usepackage{microtype}      
\usepackage{xcolor}         
\usepackage{mathtools}
\usepackage{subcaption}
\usepackage{mwe}
\usepackage{xcolor}
\usepackage{wrapfig}
\usepackage{graphicx}
\usepackage{amssymb, amsmath, amsthm, amsfonts}
\usepackage{cleveref}
\usepackage{shortcuts}
\usepackage{amsmath, amsthm, thm-restate}
\usepackage{bbold}

\usepackage[textsize=tiny]{todonotes}

\icmltitlerunning{Conformalization of Sparse Generalized Linear Models}

\begin{document}

\twocolumn[
\icmltitle{Conformalization of Sparse Generalized Linear Models}




\begin{icmlauthorlist}
\icmlauthor{Etash Kumar Guha}{schcs}
\icmlauthor{Eugene Ndiaye}{apple}
\icmlauthor{Xiaoming Huo}{schie}
\end{icmlauthorlist}

\icmlaffiliation{schcs}{College of Computing, Georgia Institute of Technology, Atlanta, GA, USA}
\icmlaffiliation{schie}{H. Milton Stewart School of Industrial and Systems Engineer- ing, Georgia Institute of Technology, Atlanta, GA, USA}
\icmlaffiliation{apple}{
Apple (Work partly done while at Georgia Tech)}
\icmlcorrespondingauthor{Etash Guha}{etash@gatech.edu}

\icmlkeywords{Conformal Prediction, Linear Models, Sparsity}

\vskip 0.3in
]



\printAffiliationsAndNotice{}  

\begin{abstract}
Given a sequence of observable variables $\{(x_1, y_1), \ldots, (x_n, y_n)\}$, the conformal prediction method estimates a confidence set for $y_{n+1}$ given $x_{n+1}$ that is valid for any finite sample size by merely assuming that the joint distribution of the data is permutation invariant. Although attractive, computing such a set is computationally infeasible in most regression problems. Indeed, in these cases, the unknown variable $y_{n+1}$ can take an infinite number of possible candidate values, and generating conformal sets requires retraining a predictive model for each candidate. In this paper, we focus on a sparse linear model with only a subset of variables for prediction and use numerical continuation techniques to approximate the solution path efficiently. The critical property we exploit is that the set of selected variables is invariant under a small perturbation of the input data. Therefore, it is sufficient to enumerate and refit the model only at the change points of the set of active features and smoothly interpolate the rest of the solution via a Predictor-Corrector mechanism. We show how our path-following algorithm accurately approximates conformal prediction sets and illustrate its performance using synthetic and real data examples.
\looseness=-1
\end{abstract}

\section{Introduction}

Modern statistical learning algorithms perform remarkably well in predicting an object based on its observed characteristics. In terms of AI safety, it is essential to quantify the uncertainty of their predictions. More precisely, after observing a finite sequence of data $\Data_n = \{(x_1, y_1), \ldots, (x_n, y_n)\}$, it is interesting to analyze to what extent one can build a confidence set for the next observation $y_{n+1}$ given $x_{n+1}$.

\looseness=-1
A classical approach is to adjust a prediction model $\mu_{\Data_n}$ on the observed data  $\Data_n$ and consider an interval centered around the prediction of $y_{n+1}$ when the fitted model receives $x_{n+1}$ as new input, \ie using $\mu_{\Data_n}(x_{n+1})$. We calibrate the confidence interval to satisfy a $100(1-\alpha)\%$ confidence by considering, for any level $\alpha$ in $(0, 1)$, the set
\begin{equation} \label{eq:naiveconfidence}
    \left\{z :\, |z - \mu_{\Data_n}(x_{n+1})| \leq Q_n(1-\alpha)\right\}\enspace,
\end{equation}
where $Q_n(1 - \alpha)$ is the $(1-\alpha)$-quantile of the empirical cumulative distribution function of the fitted residuals
$|y_i - \mu_{\Data_n}(x_{i})|$ for indices $i$ in $\{1, \ldots, n\}$. If the fitted model is close to the exact value, this method is approximately valid as $n$ goes to infinity.\looseness=-1

Alternatively, conformal prediction is a versatile and simple method introduced in \citep{Vovk_Gammerman_Shafer05, Shafer_Vovk08} that provides a finite sample and distribution free $100(1 - \alpha)\%$ confidence region for the predicted object based on past observations. The main idea is to follow the construction of the confidence set in \Cref{eq:naiveconfidence} by using candidate values for $y_{n+1}$. Since the true $y_{n+1}$ is not given in the observed dataset $\Data_n$, one can instead learn a predictive model $\mu_{\Data_{n+1}(z)}$ on an augmented database 
$$\Data_{n+1}(z) = \Data_n \cup (x_{n+1}, z) \enspace,$$ 
where a candidate $z$ replaces the unknown response $y_{n+1}$. We can, therefore,  define a prediction loss for each observation and rank them. A candidate $z$ will be considered conformal or typical if the rank of its loss is sufficiently small. The conformal prediction set will simply contain the most typical $z$ as a confidence set for $y_{n+1}$.
More formally, the conformal prediction set is obtained as
\begin{equation}\label{eq:conformal_confidence}
    \left\{z :\, |z - \mu_{\Data_{n+1}(z)}(x_{n+1})| \leq Q_{n+1}(1 - \alpha, z)\right\} \enspace,
\end{equation}
where $Q_{n+1}(1 - \alpha, z)$ is the $(1-\alpha)$-quantile of the empirical cumulative distribution function of the refitted residuals, e.g., $|y_i(z) - \mu_{\Data_{n+1}(z)}(x_{i})|$
for indices $i$ in $\{1, \ldots, n+1\}$ and $y(z)=(y_1, \ldots, y_n, z)$. This method benefits from a strong coverage guarantee without any assumption on the distribution, including finite sample size $n$; see \Cref{sec:Conformal_Prediction}.
The conformal prediction approach has been applied for designing uncertainty sets in active learning \citep{Ho_Wechsler08}, anomaly detection \citep{Laxhammar_Falkman15, Bates_Candes_Lei_Romano_Sesia21}, few-shot learning \citep{Fisch_Schuster_Jaakkola_Barzilay21}, time series \citep{Chernozhukov_Wuthrich_Zhu18, Xu_Xie20, Chernozhukov_Wuthrich_Zhu21, Lin_Trivedi_Sun22}, or to infer the performance guarantee for statistical learning algorithms \citep{Holland20, Cella_Martin20, Ndiaye22}. We refer to the extensive reviews in \citet{Balasubramanian_Ho_Vovk14} for other applications to artificial intelligence. 
Despite its attractive properties, the computation of conformal prediction sets traditionally requires fitting a model $\mu_{\Data_{n+1}(z)}$ for each possible augmented dataset $\Data_{n+1}(z)$ corresponding to each possible candidate $z$ for $y_{n+1}$. The number of possible candidates is infinite in a regression setting where an object can take an uncountable number of possible values. 
Therefore, the computation of conformal prediction is generally infeasible without additional structural assumptions on the underlying model fit. Otherwise, the calculation costs remain high or impossible. While many algorithms encounter this problem of fitting many models under alterations to the regularization parameter $\lambda$ \citep{Park_Hastie07}, to our knowledge, such algorithms do not exist for general loss functions under changes to the dataset without high computation cost. We can avoid the central issue of refitting the model many times by using the structural assumptions given by the setting of General Linear Models with $\ell_1$ regularization. \looseness=-1



\paragraph{Contributions} We generalize linear homotopy approaches from quadratic loss to a broader class of nonlinear loss functions using numerical continuation to efficiently trace a piecewise smooth solution path. Overall, we propose a homotopy drawing algorithm that efficiently keeps track of the weights over the space of possible candidates using the sparsity induced by the $\ell_1$ regularization. We develop an efficient Conformal Prediction algorithm for sparse generalized linear models from this homotopy algorithm. Additionally, using numerical continuation and the patterns in the sparsity of the weights, we relinquish the expensive necessity of retraining the model many times from random initialization. Furthermore, we provide a primal prediction step that significantly reduces the number of iterations needed to obtain an approximation at high precision. We illustrate the performance of our algorithm as a homotopy drawer and a conformal set generator using Quadratic, Asymmetric and Robust Loss functions with $\ell_1$ regularization.  \looseness=-1

\paragraph{Related Works}
 Our methodology uses numerical continuation (also called homotopy) to generate a path of solutions. Such continuation techniques have been previously used when the objective function is differentiable \citep{Allgower_Georg12}, \citep{Hastie_Rosset_Tibshirani_Zhu04} for support vector machine, \citep{Bach_Thibaux_Jordan04} for logistic regression, and more general loss functions regularized with the $\ell_1$ norm in \citep{Rosset_Zhu07, Park_Hastie07, Tibshirani13, Mairal_Yu12}. However, the latter focus on the regularization path and plot the solution curve as the regularization parameter $\lambda$ varies. To our knowledge, there does not exist work generating the solution curve as the label $z$ varies in $y(z)$ for general loss functions. In the setting we consider, we recall that it is the response vector that is parameterized as $y(z) = (y_1, \ldots, y_n, z)$ for a real value $z$; for which \citet{Garrigues_Ghaoui09} and \citet{Lei19} proposed a homotopy algorithm when the loss function is quadratic. However, such algorithms do not work for general nonlinear loss functions; our algorithm extends these works to such nonlinear loss functions. For such loss functions, works such as \citet{Ndiaye_Takeuchi19} aim to approximate the homotopy only enough to generate the conformal prediction set. However, this work suffers much worse as increasing accuracy is required when drawing the homotopy and cannot, for example, recover the path with quadratic loss, for which an exact homotopy algorithm is known.

\paragraph{Notation}For a nonzero integer $n$, we denote $[n]$ to be the set $\{1, \cdots, n\}$. Furthermore, the row-wise feature matrix is $X = [x_1, \cdots, x_{n+1}]^\top$ such that $X \in \mathbb{R}^{(n+1) \times p}$. We use the notation $X_A$ to refer to the sub-matrix of $X$ assembled from the columns with indices in $A$. If we need to do so for only one index $j$, where $j \in [p]$, we use $X_j$. For brevity, we will define $\sigma_{\text{max}}(X_A)$ as the maximum singular value of $X_A$, i.e. $\sigma_{\text{max}}(X_A) = \|X_A\|_2$. We also similarly define $\sigma_{\text{min}}(X_A)$. If a function $\beta(z)$ returns a vector for some input $z$, we can index that output vector by $\beta_A(z)$, where $A \subset [p]$ or $\beta_j(z)$ where $j \in [p]$. Moreover, given a function $f(x_i, x_j)$ of two variables, we denote the gradient of that function as $\partial f$. Furthermore, we use the simple notation $\partial_{i,j,k}f = \frac{\partial^3 f}{\partial x_i  \partial x_j\partial x_k}$ where $i,j,k \in [2]$.
We denote the smallest integer no less than a real value $r$ as $\lceil r \rceil$.
We denote by $Q_{n+1}(1 - \alpha)$, the $(1 - \alpha)$-quantile of a real valued sequence $(U_i)_{i \in [n + 1]}$, defined as the variable $Q_{n+1}(1 - \alpha) = U_{(\lceil (n+1)(1-\alpha) \rceil)}$, where $U_{(i)}$ are the $i$-th order statistics.
For $k$ in $[n+1]$, the rank of $U_k$ among $U_1, \cdots, U_{n+1}$ is defined as 
$\texttt{Rank}(U_k) = \sum_{i=1}^{n+1}\mathbb{1}_{U_i \leq U_k}$.


\section{Sparse Generalized Linear Models}

By definition of the conformal prediction set in \Cref{eq:conformal_confidence}, one needs to consider an augmented dataset $\Data_{n+1}(z)$ for any possible replacement of the target variable $y_{n+1}$ by a real value $z$. This implies the computation of the whole path $z \mapsto \mu_{\Data_{n+1}(z)}(x_{n+1})$ as well as the path of scores and quantiles.
However, it is generally difficult to achieve.
We focus on the Generalized Linear Model (GLM) regularized with an $\ell_1$ norm that promotes sparsity of the model parameter.
For a fixed $z \in \bbR$, the weight $\beta^\star(z)$ is defined as a solution to the following optimization problem
\begin{equation}\label{eq:opt_problem}
\beta^\star(z) \in \argmin_{\beta \in \mathbb{R}^p} f(y(z),X\beta) + \lambda \norm{\beta}_1 \enspace.
\end{equation}
where the data fitting term $f(y(z), y^{\star}(z))$ is a non negative loss function between a prediction $y^{\star}(z)$ and the augmented vector of labels $y(z) = (y_1, \cdots, y_n, z).$
We parameterize a linear prediction as 
$y_i^{\star} = x_{i}^{\top} \beta^{\star}(z)$
and the empirical loss is
\begin{equation*}
f(y(z), y^\star(z)) =
\sum_{i=1}^{n} \ell(y_i, y_i^{\star}(z)) + \ell(z, y^{\star}_{n+1}(z)) \enspace.
\end{equation*}
There are many examples of cost functions in the literature. A popular example is the power norm regression, where $\ell(a, b) = |a - b|^q$. When $q=2$, this corresponds to the classical linear regression. The cases where $q = [1, 2)$ are frequent in robust statistics, where the case $q = 1$ is known as the least absolute deviation. One can also consider the loss function \texttt{Linex} \citep{Gruber10, Chang_Hung07} which provides an \texttt{asymmetric} loss function $\ell(a, b) = \exp(\gamma(a - b)) - \gamma(a - b) - 1$, for $\gamma \neq 0$.\looseness=-1

\subsection{Assumptions and Properties}
We first describe the structure of the optimal solution $\beta^\star(z)$ for a candidate $z$. A solution to the optimization problem from \Cref{eq:opt_problem} must obey the first-order optimality condition. Analyzing the solution reveals a set of weights in $\beta^\star(z)$ whose value is $0$ and, thus, does not contribute to the inference. This is a crucial property of $\ell_1$ regularization.

\begin{restatable}[]{lem}{uniquesolution}\label{lem:uniqueness_solution}
A vector $\beta^\star(z) \in \mathbb{R}^p$ is optimal for \Cref{eq:opt_problem}  if and only if for $y^\star(z) = X\beta^\star(z)$, it holds
\begin{equation}\label{eq:fermat_equality}
    -X^{\top} \partial_2 f(y(z),y^\star(z)) = \lambda v(z) \enspace,
\end{equation}
where $v(z)$ belongs to the subdifferential of the $\ell_1$ norm at $\beta^\star(z)$ \ie $\forall j \in \{1, \ldots, p\}$, we have
\begin{align}\label{eq:subdifferential_l1norm}
v_j(z) &\in
\begin{cases} 
      \{\sign(\beta_j^{\star}(z))\} &\text{ if } \beta_j^{\star}(z) \neq 0 \enspace, \\
      [-1, 1] &\text{ if } \beta_j^{\star}(z) = 0 \enspace.
   \end{cases}
\end{align}
\end{restatable}

Within this lemma, we wish to formally distinguish between nonzero weights and zero weights, as this helps determine the value of $v_j(z)$, per \Cref{eq:subdifferential_l1norm}.

\begin{restatable}[]{defi}{activeset}
We define our active set at a point $z$ as 
\begin{equation}\label{eq:active_set}
A(z) = \left\{j \in [p]:\; |X_{j}^{\top}\partial_2 f(y(z), y^\star(z))| = \lambda \right\}\enspace.
\end{equation}

The active set contains at least all the indices of the optimal solution that are guaranteed to be nonzero. We will denote $A=A(z)$ if there is no ambiguity.
\end{restatable}
%
The following result provides sufficient conditions to ensure \emph{uniqueness of the solution path}, \ie for any $z$, there exists a single optimal solution $\beta^\star(z)$ for Problem \ref{eq:opt_problem}.

\begin{restatable}[]{lem}{uniquepath}
\label{lem:unique_path}
For all $z$, we assume that the matrix $X_{A(z)}$ is full rank and that the loss function $f$ is strictly convex. With these two assumptions, for all candidates $z$, only one unique optimal solution $\beta^\star(z)$ exists. Thus, the solution path $z \mapsto \beta^\star(z)$ is well defined. 
\end{restatable}

In the following, for simplicity of the presentation of the algorithms, we will add the classical qualification condition that the active set coincides with the support of the solution for any candidate $z$ where the path is differentiable.

\section{Efficient Computation of the Solution Path}

We aim to finely approximate the function $\beta^\star(z)$ as $\hat{\beta}(z)$ across all candidates $z$. The initial and main observation is that the active set map (resp. solution path) is piecewise constant (resp. smooth). That is to say, That is to say, the variable selected by the $\ell_1$ penalty is invariant with respect to small perturbation of the input data. Building on this, the path drawing algorithm is a combination of finding points where the active set changes occur and estimating the optimal solution, leveraging the regularity of the loss $f$. 

We have two situations for a change in the active set:
\begin{itemize}
\item A nonzero variable becomes zero \ie $\exists j \in A(z)$ s.t.
\begin{equation*}
     \beta_j^{\star}(z) \neq 0 \text{ and } \beta_j^{\star}(z_{j}^{ \rm{out}}) = 0 \enspace.
\end{equation*}
\item A zero variable becomes nonzero \ie $\exists j \in A^c(z)$ s.t.
\begin{equation*}
    |X_{j}^{\top}\partial_2 f(y(z_{j}^{\rm{in}}), y^\star (z_{j}^{\rm{in}}))| = \lambda \enspace.
\end{equation*}
\end{itemize}

Here, $z_j^{\text{out}}$ and $z_j^{\text{in}}$ are the estimated points where variable $j$ could leave or join the active set, respectively. With decreasing input $z$, the next change point occurs at
\begin{equation}\label{eq:change_point}
    z_{\rm{next}}(z) = \max\left(\max_{j \in A(z)} z_{j}^{\rm{out}},\; \max_{j \in A^c(z)} z_{j}^{ \rm{in}} \right) \enspace.
\end{equation}
Here, $z_{\rm{next}}(z)$ is the function that finds where the active set changes after point $z$. The set of change points are called \textit{kinks} of the path because they correspond to the non-differentiable points of the solution path
$z \mapsto \beta^\star(z)$. 
Core difficulties are that $f$ can be highly nonlinear, and the optimal weights $\beta^*(z^+)$ at an arbitrary point $z^+$ cannot be efficiently computed for many loss functions. To alleviate this, our algorithm sequentially creates a linearized version of $\beta_{A}^{\star}(z^+)$ called $\tilde{\beta}_A(z^+)$ (\Cref{sec:compbetaz}) in order to estimate the active set changes (\Cref{sec:activeupdate} and \Cref{sec:speedactivechange}). Given a point of active set change $z_t$, we can manually correct $\tilde{\beta}_A(z_t)$ into $\hat{\beta}_A(z_t)$ so that $\hat{\beta}_A(z_t) \approx \beta_A^\star(z_t)$ up to a negligible optimization error $\epsilon_{\text{tol}}$ using any appropriate solver (\Cref{sec:correction}).
It then approximates $\beta_{A^+}^{\star}(z^+)$, where $A^+$ is the new active set, repeating these steps until the stopping point is reached. We detail the entire pipeline in \Cref{alg:nonlinear_alg} and illustrate how our approximated solution path deviates from the exact one for different loss functions in \Cref{fig:figure1}. 
\looseness=-1

\subsection{Solution Estimation}
\label{sec:compbetaz}
We wish to approximate $\beta_{A}^{\star}(z^+)$ for a candidate $z^+$ smaller than the most recently found kink $z_t$ where $A(z^+) = A(z_t)$. To start, we will assume access to the corrected (up to negligible error) weights $\hat{\beta}_A(z_t)$ at the previous kink $z_t$. We can use a local linearization of the solution path as
\begin{equation}\label{eq:lin_inv_grad}
\tilde{\beta}_A(z^+) = \hat{\beta}_A(z_t) + \hat{\beta}_{A}^\prime (z_t)  \times  (z^+ - z_t) \enspace,
\end{equation}

where, $\hat{\beta}_{A}^\prime(z_t)$ is our approximation of the true slope $\frac{\partial\beta_{A}^{\star}}{\partial z}(z_t)$, which we do not have access to.  To understand this term, we follow \citet{Park_Hastie07} to define 
\begin{equation*}
    H(y(z), \beta_{A}^{\star}(z)) = X_A^{\top} \partial_2 f(y(z),y^\star(z)) + \lambda v_A \enspace,
\end{equation*}
From the Optimality Condition in \Cref{eq:fermat_equality}, it holds
$$H(y(z), \beta_{A}^{\star}(z)) = 0 \Longrightarrow \frac{\partial H}{\partial z} = 0 \enspace. $$
By the implicit function theorem and the chain rule, we have
\begin{align*}
    \frac{\partial \beta_{A}^{\star}}{\partial z} &= -\left(\frac{\partial H}{\partial \beta}\right)^{-1}\frac{\partial H}{\partial y}\frac{\partial y}{\partial z} \\
    \frac{\partial H}{\partial \beta} &= X_A^\top\partial_{2, 2} f(y(z), y^\star(z))X_A \\
   \frac{\partial H}{\partial y} &= X_A^\top\partial_{2, 1} f(y(z), y^\star(z))\\
   \frac{\partial y}{\partial z} &= (0, \ldots, 0, 1)^{\top} \enspace.
\end{align*}
To compute an approximation of $\frac{\partial\beta_{A}^{\star}}{\partial z}(z_t)$, we use a plug-in approach and only replace the (unknown) exact value of $y^\star(z_t) = X\beta^\star(z_t)$ with the approximate $\hat{y}(z_t) = X\hat \beta(z_t)$, yielding $\hat{\beta}_{A}^{\prime}(z_t)$. Notably, we get an equation for $\tilde{\beta}_A(z^+)$, which is efficient to compute given $y(z^+)$. As a reminder, the loss function $f$ differentiates this algorithm from existing path-finding algorithms tailored for changes in the hyperparameter $\lambda$. If $f$ is the Quadratic loss function, we recover the path-finding algorithm from \citet{Lei19}. A completely different homotopy will be generated if it is another loss function. \looseness=-1

\subsection{Active Set Updates}
\label{sec:activeupdate}
We have to track the changes that may occur in the active sets along the path sequentially depending on whether the variable leaves or enters the active set.
We will compute our path restricted in the interval $[z_{\min}, z_{\max}]$ where
$$z_{\min} = \min(y_1, \dots, y_n) \text{ and } z_{\max} = \max(y_1, \dots, y_n)\enspace.$$ For sufficiently large sample size $n$, any point $z$ outside this interval has a very low probability of being in the conformal set since it is an outlier of a label; see justification in \Cref{lem:justified}. For simplicity, we reiterate  that we know the corrected $\hat\beta(z_t)$ at the most recent kink $z_t$ approximating  $\beta^*(z_t)$ up to error $\epsilon_{\text{tol}}$ and the active set of weights $A(z_t)$. We estimate the kinks by following \Cref{eq:change_point} and replacing the exact solution $\beta^{\star}(z_t)$ by $\tilde{\beta}(z_t)$ in \Cref{eq:lin_inv_grad}. 

As such, we will iteratively set $z_{t+1} = z_{\rm{next}}(z_t)$ as the next change point following \Cref{eq:change_point}.

\paragraph{Leaving the active set} \label{para:join_active} At the point, where a nonzero variable becomes zero, we know that by \Cref{eq:lin_inv_grad}, we have a closed form approximation of $\beta_{A}^{\star}(z^+)$ given $\beta_A(z_t)$. Therefore, for a feature index $j \in A$, we have a closed-form approximation for $\beta_j^{\star}(z^+)$ in terms of $z^+$, which we can compute efficiently. Thus, from \Cref{eq:lin_inv_grad}, $j$ leaving the active set occurs at $\beta_j^{\star}(z^+) = 0$ implies a kink occurs at $z^+$ when $0 \approx \tilde{\beta}_j(z^+)$ defined in the R.H.S. of \Cref{eq:lin_inv_grad}; which is easily solvable in closed-form. Thus, for an active variable $j$ with nonvanishing gradient $\hat \beta_{j}^{\prime}(\hat z_t) \neq 0$, we define
$$z_{j, t+1}^{\rm{out}} = z_t -  \frac{\hat{\beta}_j(z_t)}{\hat \beta_{j}^{\prime}( z_t)} \enspace,$$ 
and define $z_{j, t+1}^{\rm{out}} = -\infty$ otherwise. We remind the reader that $\hat \beta_{j}^{\prime}(z_t)$ is our approximation of the true slope $ \frac{\partial \beta_{j}^{\star}}{\partial z}(z_t)$ from \Cref{sec:compbetaz}.

\paragraph{Joining the active set} \label{para:leave_active} At the point where a variable becomes nonzero, we know from \Cref{eq:fermat_equality} that for any inactive variable $j \in A^c$ that joins the active set
$$\lvert X_{j}^{\top} \partial_2 f(y(z^+), X_{A^+} \beta_{A^+}^{\star}(z^+))\rvert = \lambda$$
where $A^+ = A \cup \{j\}$. 
However, given that we are searching for a point $z^+$ where the active sets shift from $A$ to $A^+$, at point $z^+$, $\beta_{j}^{\star}(z^+)$ is roughly $0$ since it is the first point where $\beta_{j}^{\star}(z^+)$ becomes nonzero. Therefore, given this information, the prediction $X_j\beta_{j}^{\star}(z^+) = 0$ where $z^+$ is a kink. Using this idea, we can provide the equivalence
$$X_{A^+}\beta_{A^+}^{\star}(z^+) = X_{A}\beta_{A}^{\star}(z^+) = y^\star(z^+) \enspace.$$
 This equivalence is useful as we know how to approximate $\beta_{A}^{\star}(z^+)$, and therefore $y^\star(z^+)$,  efficiently from \Cref{eq:lin_inv_grad}. Therefore, the $j$-th variable must join the active set at approximately $z^+$ such that $\mathcal{I}_j(z^+) = 0$ where
\begin{equation}
    \label{eq:invariable}
   \mathcal{I}_j(z^+) = \lvert X_{j}^{\top}\partial_2 f(y(z^+), y^\star(z^+))\rvert - \lambda  \enspace. 
\end{equation}
We also leverage a plug-in estimate of \Cref{eq:invariable} by replacing $y^\star(\cdot)$ by $\hat y(\cdot)$.
We could use a root-finding function to efficiently find the roots of the function $\mathcal{I}_j(z^+)$ where the kink may lie. However, we seek a closed form as in \Cref{eq:lin_inv_grad} to make finding the roots of $\mathcal{I}_j(z^+)$ more efficient. We do this via linearization again. 

\subsection*{Approximation of $\partial_2 f(y(z^+), y^\star(z^+))$}
\label{sec:speedactivechange}
While $\tilde{\beta}_j(z^+)$ is linear in $z^+$, giving way to an explicit solution for $z^+$, this property does not hold for $\mathcal{I}_j(z^+)$ in \Cref{eq:invariable}. To achieve such a form, we need to linearize further $\partial_2 f(y(z^+), y^\star(z^+))$.
To simplify, we denote
$$f(y(z), y^\star(z)) = f\circ\zeta(z) \text{ where } \zeta(z) = (y(z), y^\star(z)) \enspace,$$ 
and approximate its gradient $\partial_2 f\circ \zeta(z^+)$ as 
%
\begin{equation}
\partial_2 f\circ \zeta(z) + \partial_{2,1} f\circ\zeta(z)^\top \Delta y + \partial_{2,2} f\circ\zeta(z)^\top \Delta y^\star \label{lin:secondgrad}
\end{equation}
where $\Delta y = y(z^+) - y(z)$ and
$\Delta y^\star = y^\star(z^+) - y^\star(z)$. We still have that \Cref{lin:secondgrad} can be nonlinear since $\Delta y^\star$ can be nonlinear in $z^+$. To alleviate this, we  leverage the local approximation of the solution path in \Cref{eq:lin_inv_grad} and the plug-in replacement of $\frac{\partial \beta_A^\star}{\partial z}$ with $\hat\beta_A^\prime$. As such, we can estimate the root of $\mathcal{I}_j(z^+)$ and sequentially define the next point where the $j$th variable becomes active. To simplify the expression, we set $\hat \zeta(z) = (y(z), \hat y(z))$ and
$$g(z_t) = [\partial_{2 1} f\circ \hat\zeta(z_t)]_{n+1} + \partial_{2,2}f\circ \hat \zeta(z_t)^\top X_A \hat\beta_A^\prime(z_t) \enspace.$$
A zero variable $j$ is estimated to become nonzero at
$$z_{j, t+1}^{\rm{in}} = z_t +  \frac{-X_{j}^{\top} \partial_2 f \circ \zeta (z_t) \pm \lambda}{ X_{j}^{\top} g(z_t) } \enspace,$$

The detailed computations are provided in \Cref{sec:detailspartf}. Note that when the denominator $g(z_t)$ is zero, we set $z_{j, t+1}^{\rm{in}} = -\infty$. Finally, the next kink is estimated as
$$ z_{t+1} = \max\left( \max_{j \in A(z_t)} z_{j, t+1}^{\rm{out}}, \max_{j \in A^c(z_t)} z_{j, t+1}^{\rm{in}} \right) \enspace.$$

\subsection{Solution Updates}
\label{sec:correction}
Our active set change point finder obtains the next kink $z_{t+1}$ by tracking all variables in the optimal solution to see whether or not it cancels out after $z_{t}$. However, our kink-finding tool requires exact knowledge of $\hat{\beta}(z_t)$, as in \Cref{eq:lin_inv_grad}. To find the next kink, we, therefore, need to know $\hat{\beta}(z_{t+1})$. To ensure that our linearized version $\tilde{\beta}_A(z_{t+1})$ is close enough to the exact solution $\beta_{A}^{\star}(z_{t+1})$, we manually correct our linearized weights $\tilde{\beta}_A(z_{t+1})$, creating our $\hat{\beta}_A(z_{t+1})$. We use the Predictor-Corrector strategy described below \citep{Allgower_Georg12}.  \looseness=-1

\paragraph{Predictor}

To initialize the solving process for $\hat{\beta}(z_{t+1})$, we first provide our linearized version $\tilde{\beta}(z_{t+1})$ from \Cref{eq:lin_inv_grad} as a warm start initialization. This vastly improves the computation time of our corrector step here after.

\paragraph{Corrector}

 The solution obtained in the warm start often has a reasonably small approximation error. For example, in the case of the Quadratic loss, this warm start is exact and correction is unnecessary. However, it generally is an imprecise estimate of the exact solution. To overcome this, we use an additional corrector step using an iterative solver, such as proximal gradient descent initialized with the predictor output, or more advanced solvers such as \texttt{CVXPY} \citep{Diamond_Boyd16} or \texttt{SKGLM} \citep{skglm}. This takes our linearized weight estimates of $\tilde{\beta}(z_{t+1})$ and outputs our approximate weights $\hat{\beta}(z_{t+1}) \approx \beta^*(z_{t+1})$ up to error $\epsilon_{\text{tol}}$ which is a hyperparameter for our corrector.

Finally, we can summarize our approximation of the homotopy as the following.
$$\hat{\beta}(z)= \begin{cases} \tilde{\beta}(z) \text{ if } z \notin \{z_1, \dots, z_t\}\\ \tilde{\beta}^\star(z)\text{ if } z \in \{z_1, \dots, z_t\} \text{ (output of corrector)}\\ \end{cases}$$
For point $z$ that is not a kink, we form our estiamte weights simply through the linearization. Otherwise, we can use the output of the corrector as our estimates. 

\begin{algorithm}[h]
\caption{Full Homotopy Generation}\label{alg:nonlinear_alg}
\begin{algorithmic}
\STATE \textbf{Input Data:} $\{(x_1, y_1), \ldots, (x_n, y_n)\}$, $x_{n+1}$, $\lambda > 0$
\STATE \textbf{Initialization: } $t = 0$
\STATE $z_0 = \max(y_1, \ldots, y_n) \quad y(z_0) = (y_1, \ldots, y_n, z_0)$
\STATE $\beta^\star(z_0) = \argmin_{\beta \in \mathbb{R}^p} f(y(z_0), X\beta) + \lambda \|\beta\|_1$
\STATE $A(z_0) = \left\{j \in [p]:\; |X_{j}^{\top}\partial_2 f(y(z_0),y^\star(z_0))| = \lambda\right\} $   
\WHILE{$z_t > \min(y_1, \ldots, y_n)$}
    \STATE $z_\mathcal{I} = \underset{j \in A^c(z_t)}{\max} z_{j, t+1}^{\rm{in}} \text{ and } j_\mathcal{I} = \underset{j \in A^c(z_t)}{\argmax} \; z_{j, t+1}^{\rm{in}}$
    \STATE $z_\mathcal{O} = \underset{j \in A(z_t)}{\max} z_{j, t+1}^{\rm{out}} \text{ and }j_\mathcal{O} = \underset{j \in A(z_t)}{\argmax} \; z_{j, t+1}^{\rm{out}}$
    \IF{$z_\mathcal{I} > z_\mathcal{O}$}
        \STATE $z_{t+1} = z_\mathcal{I}$
        \STATE $A(z_{t+1}) = A(z_t) \cup \{j_\mathcal{I}\}$
    \ELSE
        \STATE $z_{t+1} = z_\mathcal{O}$
        \STATE $A(z_{t+1}) = A(z_t) \setminus \{j_\mathcal{O}\}$
    \ENDIF
    \STATE \textbf{Predictor}
    $$\tilde{\beta}(z_{t+1}) = \hat{\beta}(z_t) + \hat{\beta}^{\prime}(z_t) \times (z_{t+1} - z_t)$$
    \STATE \textbf{Corrector} warm started with $\tilde{\beta}(z_{t+1})$
    \STATE $$\hat{\beta}(z_{t+1}) = \argmin_{\beta \in \mathbb{R}^p}f(y(z_{t+1}), X\beta) + \lambda \|\beta\|_1$$
    \STATE $t = t + 1$
\ENDWHILE\\
\STATE $\textbf{RETURN: } \hat{\beta}(z_t), z_t \text{ for all } t$
\end{algorithmic}
\end{algorithm}

\section{Conformal Prediction for Sparse GLM}
\label{sec:Conformal_Prediction}
Given a homotopy for specific data and loss function, computing the Conformal Prediction set relies on a simple calculation using the homotopy. Meanwhile, the primary tool for proving its validity is that the rank of one variable among an exchangeable and identically distributed sequence follows a (sub)-uniform distribution \citep{Brocker_Kantz11}.\looseness=-1

This idea of rank helps construct distribution-free confidence intervals. We can estimate the conformity of a given candidate $z$ by calculating its prediction loss $\lvert z - y^\star_{n+1}(z)\rvert$ and compute its rank relative to the losses of the other datapoints. The candidate will be considered conformal if the rank of its loss is sufficiently small.
Let us define the conformity measure for $\Data_{n+1}(z)$ as
\begin{align}\label{eq:arbitrary_conformity_measure}
E_{i}(z) &= |y_i - y_i^{\star}(z)|,\quad \forall i \in [n] \enspace,\\
E_{n+1}(z) &= |z - y_{n+1}^{\star}(z)| \enspace.
\end{align}
The main idea for constructing a conformal confidence set is to consider the conformity of a candidate point $z$ measured as\looseness=-1
\begin{equation}\label{eq:general_def_of_pi}
\pi(z) = 1 - \frac{1}{n+1} \texttt{Rank}(E_{n+1}(z)) \enspace.
\end{equation}
The conformal prediction set will collect the most conformal $z$ as a confidence set for $y_{n+1}$, \ie gathers all the real values $z$ such that $\pi(z) \geq \alpha$. This condition occurs if and only if the score $E_{n+1}(z)$ is ranked no higher than $\lceil(n+1)(1 - \alpha)\rceil$, among the sequence $\{E_{i}(z)\}_{i \in [n + 1]}$, \ie
\begin{align*}
 \left\{z \in \bbR:\, E_{n+1}(z) \leq Q_{n+1}(1 - \alpha, z)\right\} \enspace,
\end{align*}
which is exactly the conformal set defined in \Cref{eq:conformal_confidence}. We need to calculate the piecewise constant function $z \mapsto \pi(z)$ to compute a conformal set. Fortunately, our framework directly sheds light on the computation of this value over the range space.

Access to the homotopy, as well as the kinks, yields an efficient methodology for calculating the conformal prediction set over the range space. Once can readily use a root-finding approach \citep{Ndiaye_Takeuchi21} but it requires the assumption that the conformal set is an interval.
Instead, we do so by tracking where changes in this set occur. Naturally, changes in the rank function only occur when the error of one example surpasses or goes below that of the error of the last example. Formally, this can be seen when 
\begin{equation}
\label{eq:where_gamma_changes}
\lvert y_i(z) - y_{i}^{\star}(z)\rvert = \lvert y_{n+1}(z) -  y_{n+1}^{\star}(z))\rvert \enspace.
\end{equation}
We will look between the two kinks to efficiently find points satisfying \Cref{eq:where_gamma_changes}. For a point $z$ between two kinks, we can efficiently estimate $y^\star(z)$. Indeed, given a point $z$ is between two kinks $z_t$ and $z_{t+1}$ with an active set $A$, we can use \Cref{eq:lin_inv_grad} to estimate the quantity $y(z) - y^\star(z)$ as 
\begin{equation*}
    \mathcal{F}(z) = y(z) - \hat{y} (z_t) + X\hat{\beta}_A^{\prime}(z_t) \times (z - z_t) \enspace,
\end{equation*}
where $\hat{\beta}_A(z_t)$ is stored from the corrector step at the kink $z_t$. Given that this value is linear in $z$, we can form a closed-form explicit approximation for what $z$ solves \Cref{eq:where_gamma_changes}. Therefore, we can look for where the $\pi(z)$ value changes between every sequential pair of kinks. To find the conformal set, we track the changes $\pi(z)$ and recompute it along each root of \Cref{eq:where_gamma_changes}, yielding an efficient methodology to compute $\pi(z)$, and, therefore, the conformal set along the space of possible $y_{n+1}$ values.\looseness=-1

\begin{algorithm}[h]
\caption{Conformal Set Generation}\label{alg:nonlinear_alg}
\begin{algorithmic}
\STATE \textbf{Input Data:} $\{z_t, \hat{\beta}(z_t)\}_{t \in [0:T]}, \alpha \in (0, 1)$\\
\STATE \texttt{//Find where changes in $\pi$ occur}
\STATE Set $\mathcal{C} = \emptyset$
\FOR{$t \in [T], i \in [n]$}
    \IF{$\exists z^+ \text{ s.t. }  \mathcal{F}(z^+)_i = \mathcal{F}(z^+)_{n+1}$}
        \IF{$|\mathcal{F}(z_t)_{n+1}| \geq |\mathcal{F}(z_t)_{i}|$}
            \STATE $\mathcal{C} = \mathcal{C} \cup \{(z^+, -1)\}$
        \ELSIF{$|\mathcal{F}(z_t)_{n+1}| \leq |\mathcal{F}(z_t)_{i}|$}
            \STATE $\mathcal{C} = \mathcal{C} \cup \{(z^+, +1)\}$
        \ENDIF
    \ENDIF
\ENDFOR
\STATE \texttt{//Get $z$ s.t. $\texttt{Rank}(E_{n+1}(z))$ is small}
\STATE Set $\mathcal{E} = \emptyset$
\STATE $\mathcal{R} = \texttt{Rank}(E_{n+1}(z_0))$
\STATE \texttt{SORT}$(\mathcal{C}) \text{ according to first argument } z$
\FOR{$z, c \in \mathcal{C}$}
    \STATE $\mathcal{R} = \mathcal{R} + c$
    \STATE $\texttt{Rank}(E_{n+1}(z)) = \mathcal{R}$
    \IF{$\texttt{Rank}(E_{n+1}(z)) \leq \lceil (n+1)(1-\alpha) \rceil$}
        \STATE $\mathcal{E} = \mathcal{E} \cup \{z\}$
    \ENDIF
\ENDFOR
\STATE $\textbf{RETURN: } (\min(\mathcal{E}), \max(\mathcal{E}))$

\end{algorithmic}
\end{algorithm}

\section{Theoretical Analysis}
To understand where and how our algorithm fails, we provide an upper bound on the pointwise error of our algorithm. The error is mainly accumulated in the linearizations we use for estimating the solution and gradient of the loss. To form such bounds, we need assumptions on the regularity of the loss function $f$ itself and on the sequence of design matrix restricted on the active sets along the path. Namely, we will see that the derivatives of the loss function is bounded.\looseness=-1
\begin{restatable}[]{lem}{boundderivatives}
\label{lm:bound_derivatives}
The second derivatives, assumed to be continuous, of the loss function $f$ are locally bounded by data-dependent constants. Indeed, for any $z \in [z_{\min}, z_{\max}]$, we have $\beta^\star(z) \in \mathcal{B}_{\|\cdot\|_1}(0, R/\lambda)$ where $$R = \underset{z \in [z_{\min}, z_{\max}]}{\max} \, f(y(z), \mathbf{0}) \enspace.$$
By Weierstrass theorem, for any $i,j \in [2]$, we have 
$$\norm{\partial_{i,j} f \circ \zeta(z)}_{2} \leq \nu_f \enspace.$$
\end{restatable}

\begin{restatable}[]{lem}{strongconvexityalongpath}
\label{lm:strong_convexity_along_path}
We assume that the loss $f$ is $\mu_f$-strongly convex \ie $ \mu_f \coloneqq \inf_{\norm{\zeta} \leq B} \norm{\partial_{2, 2} f \circ \zeta(z)} > 0$, where $B$ is provided in the appendix.
Thus, for any $z \in [z_{\min}, z_{\max}]$, the maximum singular value of the inverse of the matrix $\frac{ \partial H}{\partial \beta} = X_A^\top\partial_{2, 2} f\circ \zeta(z) X_A$ is upper bounded as 
$$\norm{\frac{ \partial H}{\partial \beta}^{-1}}_{2} \leq \frac{1}{\sigma_{\min}^{2}(X_A) \times \mu_f} \enspace.$$ 
\end{restatable}

With these two lemmas, we can form our error bounds. 
\begin{restatable}[]{thm}{wholeerror}
\label{thm:wholeerror}
 The error between our linearized weights $\tilde{\beta}(z^{+})$ and the true weights $\beta^\star(z^{+})$ is upper bounded by 
 $$
\norm{\tilde{\beta}(z^{+}) - \beta^\star(z^{+})}_2 \leq \epsilon_{\rm{tol}} + \frac{L \nu_f}{\mu_f} \times |z^+ - z_t| \enspace.
$$
$$\text{where } L = \frac{\sigma_{\max}(X_{A(z_t)})}{\sigma_{\min}^2(X_{A(z_t)})} + 
\sup_{z \in [z^+, z_t]} \frac{\sigma_{\max}(X_{A(z)})}{\sigma_{\min}^2(X_{A(z)})} \text{,}$$
and $z_t$ is the prior kink of $z^+$.
\end{restatable}



\begin{restatable}[]{thm}{uppboundg}
\label{thm:uppboundg}
The estimation error is upper bounded by $$\norm{\partial_2 f \circ \zeta(z^+) - \partial_2 f\circ \hat{\zeta}(z^+)}_2  \leq K \left[  \epsilon_{\rm{tol}} + \frac{L\nu_f}{\mu_f}|z^+ - z_t|\right]$$
where $K=\nu_f \times \sigma_{\max}(X_A)$.
\end{restatable} 


\section{Numerical Experiments}
Our central claim is twofold. Our method efficiently and accurately generates the homotopy over general loss functions. Our method also efficiently and accurately generates conformal sets over general loss functions.
We demonstrate these two claims over different datasets and loss functions.
For reproducibility, our implementation is at \newline \url{github.com/EtashGuha/sparse_conformal}.

\paragraph{Datasets} \label{para:datasets} We use four datasets to illustrate the performance of our algorithm. The first three are real datasets sourced from \citep{scikit-learn}. The Diabetes dataset is a  regression dataset with $20$ features and $442$ samples. Additionally, we use the well-known regression dataset from \citep{Friedman91} denoted as Friedman1, which has $10$ features and $100$ samples. We also use the multivariate dataset denoted Friedman2 from \citep{breiman1996bagging}, which has $100$ samples and $4$ features. These datasets are used to demonstrate the capabilities of our algorithm on real datasets. We also generate regression problems synthetically. We sample the data and labels from a uniform distribution between $[-1,1]$. We also divide by the standard deviation to normalize the dataset. We generate two different synthetic datasets, one normal-sized dataset, denoted \texttt{synthetic} with $100$ samples and $100$ features, and a larger dataset, denoted \texttt{large} with $1000$ features and $20$ samples. This larger dataset is intended to display our algorithm's complexity in terms of the number of features. These datasets represent a reasonable range of regression problems usable for our experiments. \looseness=-1

\paragraph{Baselines}
To form a baseline for our algorithm, we use several baselines. This baseline is the most naive conformal prediction algorithm. For Grid algorithms, the algorithm selects $100$ potential candidates evenly across the range of possible candidates. It uses the primal corrector at each point to calculate the weights to form the homotopy. A more sophisticated conformal prediction and homotopy generating algorithm is the Approximate homotopy from \citep{Ndiaye_Takeuchi19}, which leverages loss function smoothness to track violations (up to a prescribed error tolerance) of the optimality condition along the path.\looseness=-1
\begin{figure*}[!ht]
        \centering
        \begin{subfigure}[b]{0.240\textwidth}
            \centering
            \includegraphics[width=\textwidth]{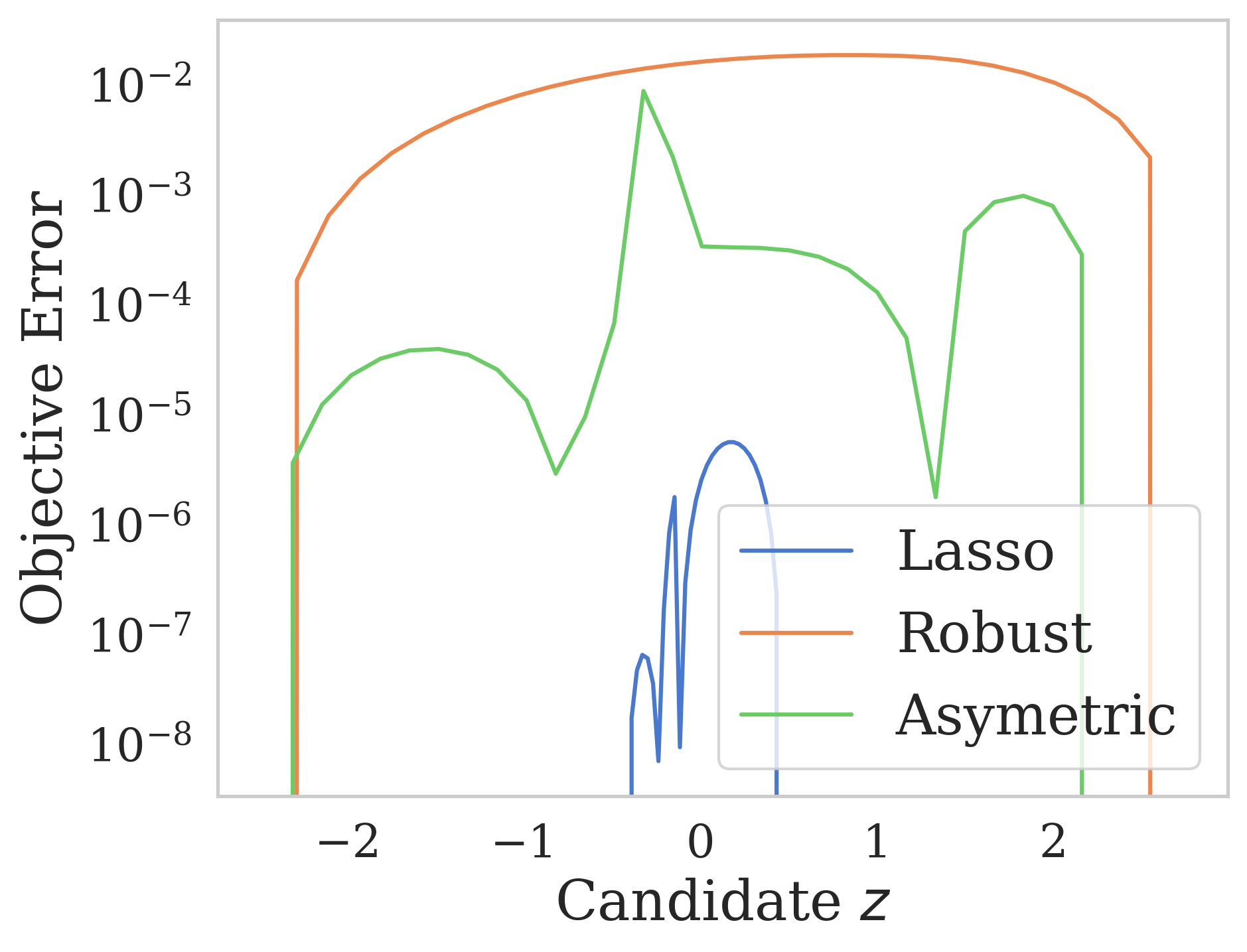}
            \caption[Network2]%
            {{\small Friedman1}}    
            \label{fig:friedman1_search}
        \end{subfigure}
        \hfill
        \begin{subfigure}[b]{0.240\textwidth}  
            \centering 
            \includegraphics[width=\textwidth]{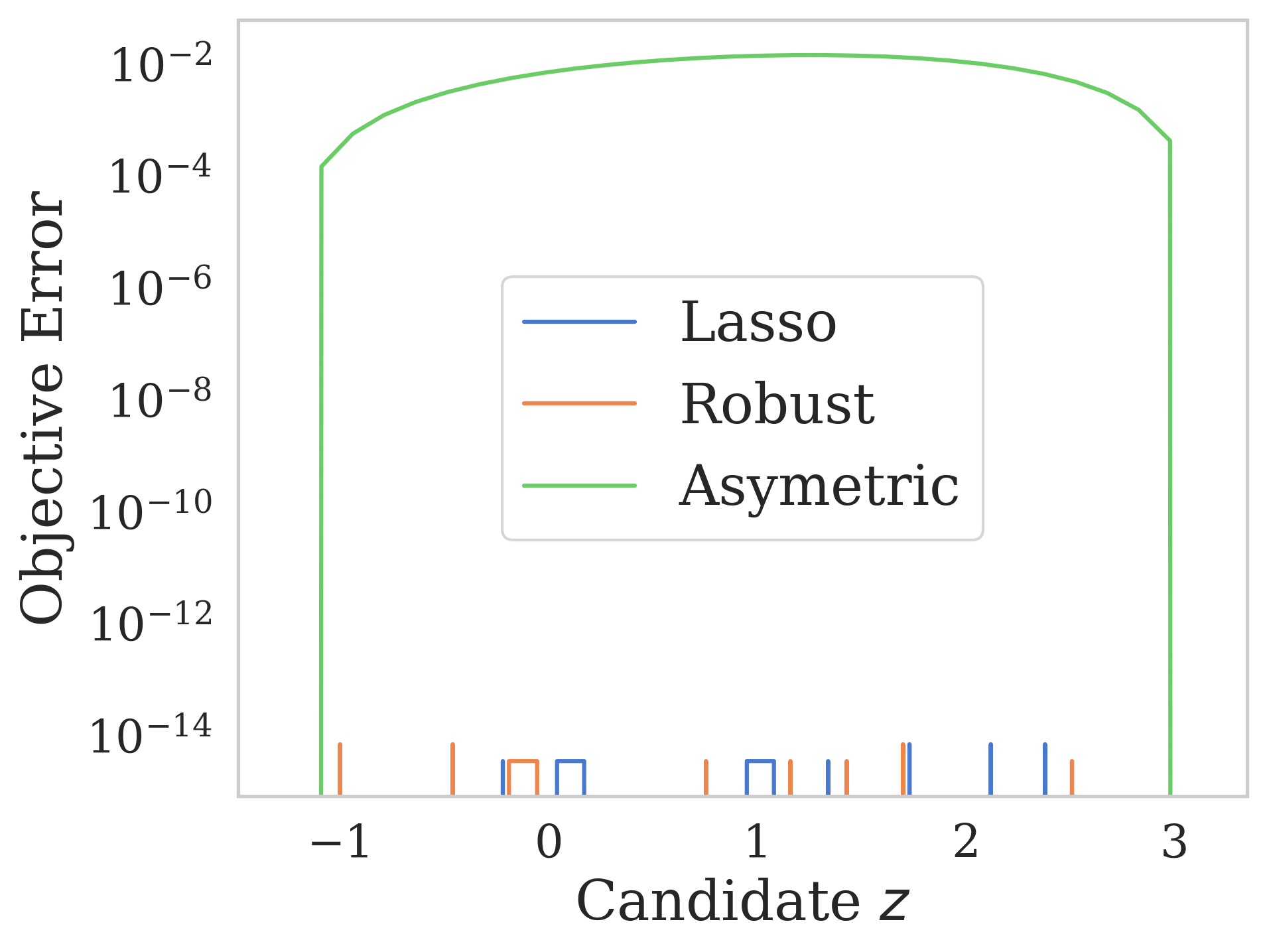}
            \caption[]%
            {{\small Friedman 2}}    
            \label{fig:friedman2_search}
        \end{subfigure}
        \begin{subfigure}[b]{0.240\textwidth}   
            \centering 
            \includegraphics[width=\textwidth]{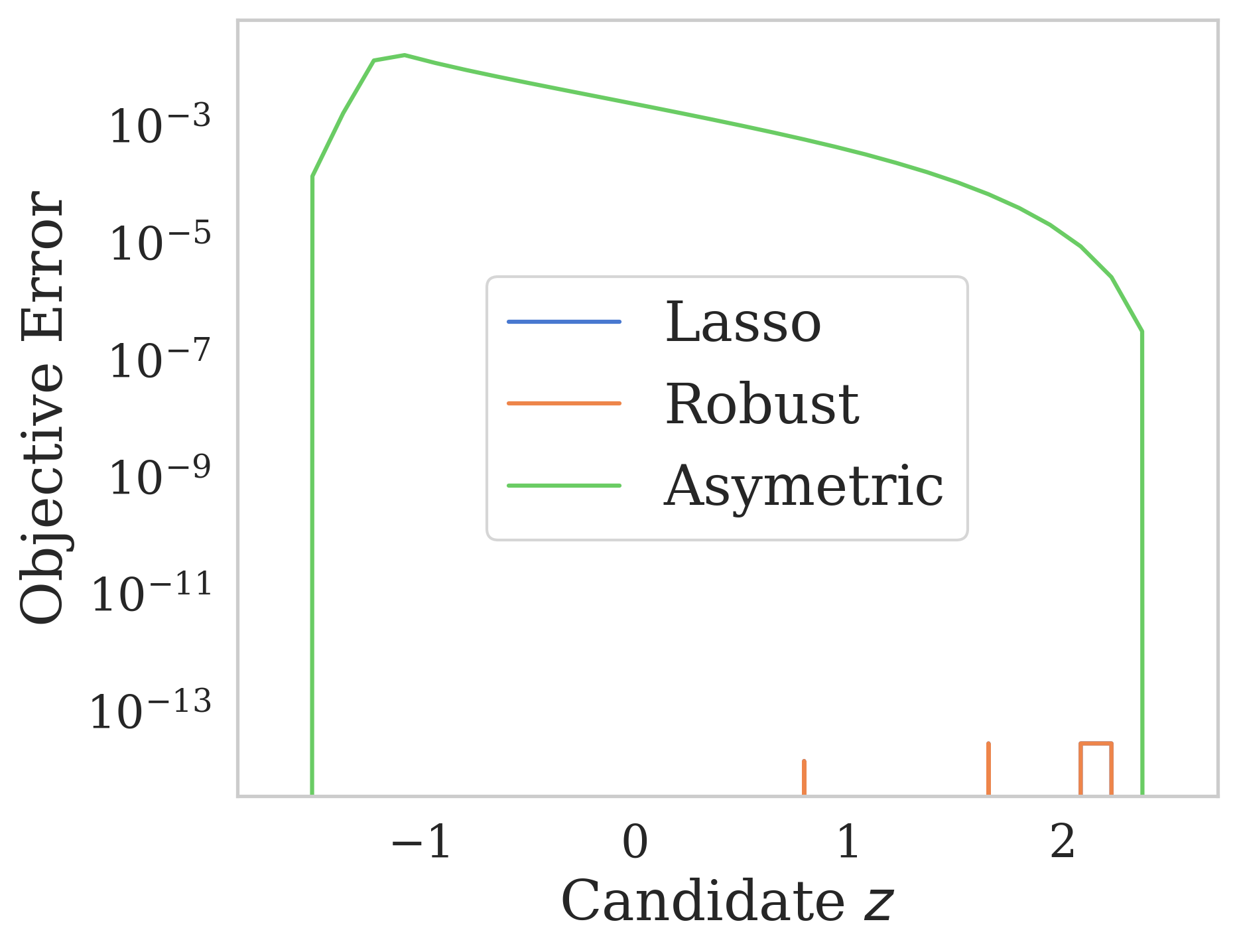}
            \caption[]%
            {{\small Diabetes}}    
            \label{fig:diabetes_search}
        \end{subfigure}
        \hfill
        \begin{subfigure}[b]{0.240\textwidth}   
            \centering 
            \includegraphics[width=\textwidth]{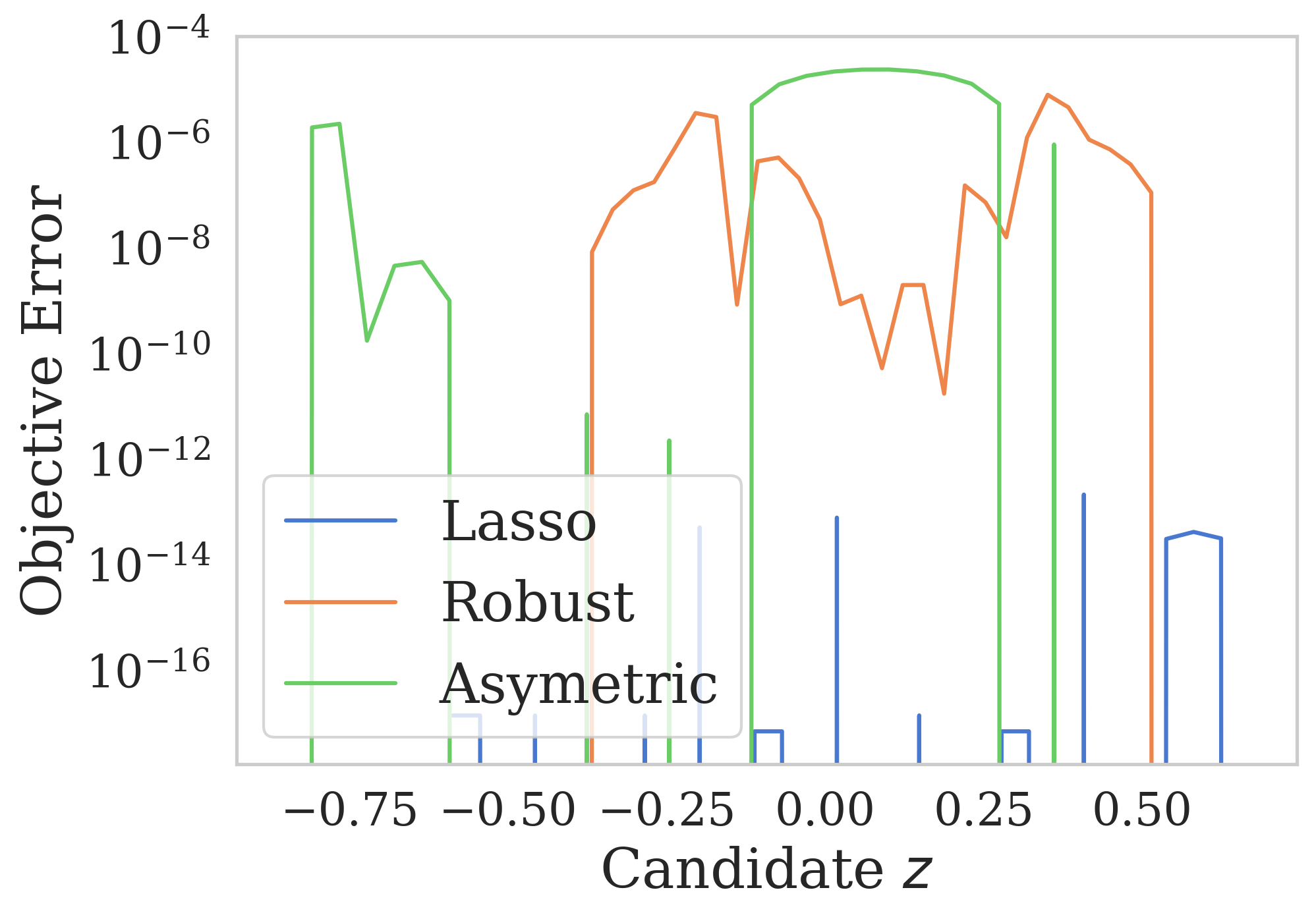}
            \caption[]%
            {{\small  Synthetic}}    
            \label{fig:synthetic_search}
        \end{subfigure}
        \caption[We demonstrate the objective error of our achieved homotopy over the space of possible $y_{n+1}$ on all 4 datasets and 3 loss functions.]
        {\small We demonstrate the objective error of our achieved homotopy over the space of possible $y_{n+1}$ on all four datasets and loss functions.} 
        \label{fig:search_over_spaces}
\end{figure*}

\begin{figure*}[!ht]
        \centering
        \begin{subfigure}[b]{0.240\textwidth}
            \centering
            \includegraphics[width=\textwidth, scale=.25]{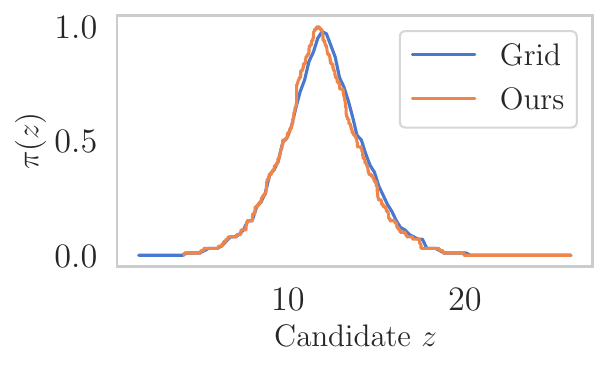}
            \caption[Network2]%
            {{\small Robust on Friedman1}}    
            \label{fig:rf1_cp}
        \end{subfigure}
        \hfill
        \begin{subfigure}[b]{0.240\textwidth}  
            \centering 
            \includegraphics[width=\textwidth]{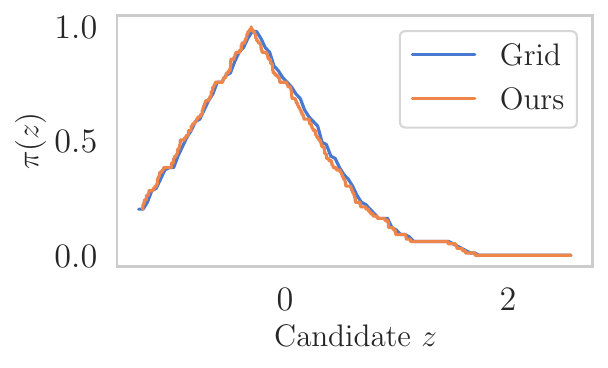}
            \caption[]%
            {{\small Robust on Friedman 2}}    
            \label{fig:rf2_cp}
        \end{subfigure}
        \begin{subfigure}[b]{0.240\textwidth}   
            \centering 
            \includegraphics[width=\textwidth]{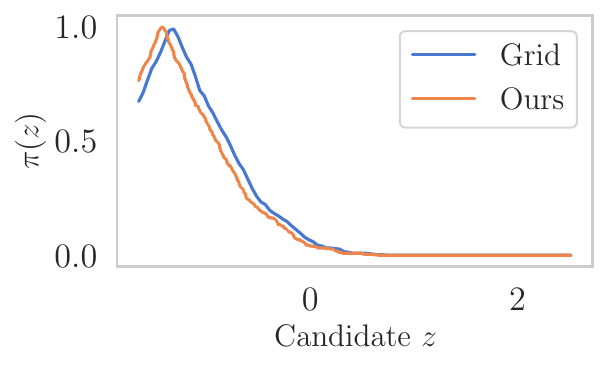}
            \caption[]%
            {{\small Robust on Diabetes}}    
            \label{fig:rd_cp}
        \end{subfigure}
        \hfill
        \begin{subfigure}[b]{0.240\textwidth}   
            \centering 
            \includegraphics[width=\textwidth]{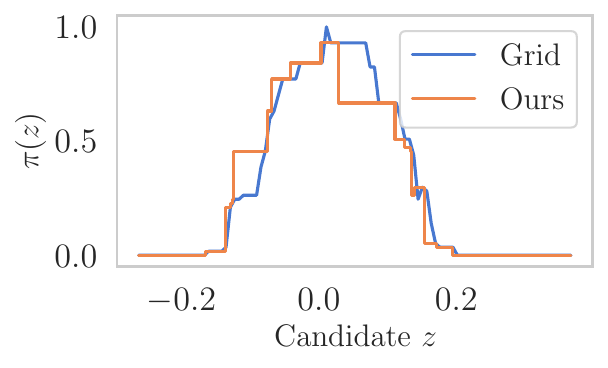}
            \caption[]%
            {{\small Robust on Synthetic}}    
            \label{fig:rsyn_cp}
        \end{subfigure}
        \caption[ The average and standard deviation of critical parameters ]
        {\small The $\pi(z)$ function as generated by a ground truth discretized searching algorithm and by our homotopy drawing algorithm for the Robust loss function over all $4$ datasets.} 
        \label{fig:mean and std of nets}
        \vskip\baselineskip
        \begin{subfigure}[b]{0.240\textwidth}
            \centering
            \includegraphics[width=\textwidth, scale=.25]{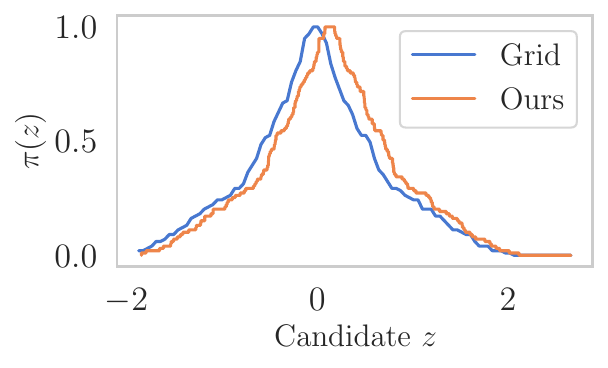}
            \caption[Network2]%
            {{\small Asymmetric on Friedman1}}    
            \label{fig:af1_cp}
        \end{subfigure}
        \hfill
        \begin{subfigure}[b]{0.240\textwidth}  
            \centering 
            \includegraphics[width=\textwidth]{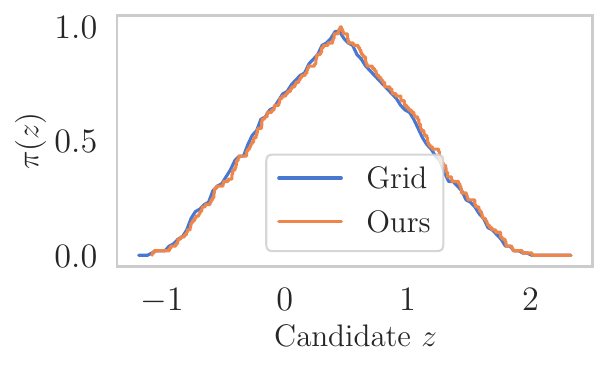}
            \caption[]%
            {{\small Asymmetric on Friedman 2}}    
            \label{fig:af2_cp}
        \end{subfigure}
        \begin{subfigure}[b]{0.240\textwidth}   
            \centering 
            \includegraphics[width=\textwidth]{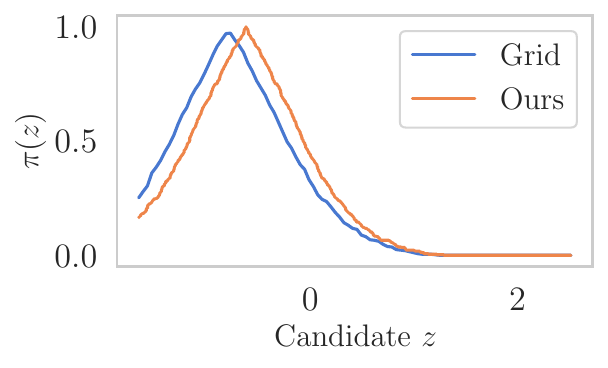}
            \caption[]%
            {{\small Asymmetric Diabetes}}    
            \label{fig:ad_cp}
        \end{subfigure}
        \hfill
        \begin{subfigure}[b]{0.240\textwidth}   
            \centering 
            \includegraphics[width=\textwidth]{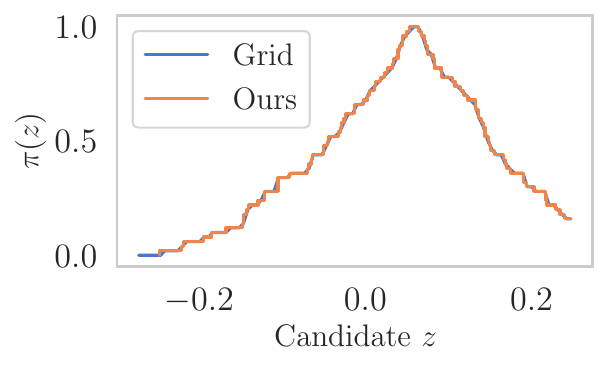}
            \caption[]%
            {{\small Asymmetric on Synthetic}}    
            \label{fig:asyn_cp}
        \end{subfigure}
        \caption[ The average and standard deviation of critical parameters ]
        {\small The $\pi(z)$ function as generated by a ground truth discretized searching algorithm and by our homotopy drawing algorithm for the Asymmetric loss function over all $4$ datasets.} 
        \label{fig:all_cp}

\end{figure*}
\subsection{Homotopy Experiments}

To test our algorithm in terms of homotopy generation, we measure our algorithm's accuracy and efficacy against different baselines across different loss functions. For all baselines and our algorithm, we use Proximal Gradient Descent for Lasso Loss and \texttt{CVXPY} for Robust and Asymmetric as Primal Correctors. Precisely, we measure the negative logarithm of the gap between primal values of the calculated $\hat \beta$ values and a ground truth baseline. We measure this gap across many possible $z$ values and take the average. The ground truth baseline is a Grid-based homotopy, where we compute the homotopy iteratively along a find grid of candidates. Given that we apply the negative logarithm to the primal gap, the larger the value reported, the smaller the true error term and the better the algorithm's performance. Moreover, we report the time taken in seconds required to form the homotopy. Our experiments cover the Lasso, Robust, and Asymmetric functions across all the datasets. \looseness=-1

We report our results in \Cref{tab:homotopy_time} and \Cref{tab:homotopy_primal-diff}. We shorten Synthetic to \texttt{Synth} and Approximate to \texttt{Appr} for brevity. As evident, we see a significant decrease in time used over Approximate Homotopy for most applications of the Lasso Loss with a significant increase in accuracy. On the largest dataset for Lasso Loss, our algorithm gets similar accuracy and is much more efficient. Furthermore, we report similar primal gaps for both ours and the approximate homotopy algorithms on Robust and Asymmetric losses. However, we achieve significant time improvements. Notably, on the Diabetes and Large dataset for Asymmetric loss and the Synthetic and Large dataset for both Asymmetric and Robust losses, we report an almost $50\%$ reduction in the time taken to achieve a similar error. Overall, across all loss types and datasets, we either achieve similar or better errors with the same or less time relative to the standard Approximate Homotopy, demonstrating the capability of our algorithm to efficiently and accurately generate the homotopy. 

To illustrate the accuracy of our algorithm, we plot the optimization error gap over the space of all $z \in [z_{\min}, z_{\max}]$ for all three loss functions and four datasets. We report the figures in \Cref{fig:search_over_spaces}. Notably, we see that on \Cref{fig:synthetic_search}, we achieve all losses better than $10^{-4}$. On other figures, all objective errors are bounded by $10^{-2}$. Our application of Lasso and Robust over all datasets achieves near $0$ objective error over the entire pass.

\begin{table}[!htp]\centering

\scriptsize
\caption{Average Time of Homotopy}
\begin{tabular}{@{}lccccl@{}}\toprule
\label{tab:homotopy_time}
& \multicolumn{5}{c}{\textbf{Dataset}} \\
\cmidrule(lr){2-6}
&\textbf{Synth. }&\textbf{Friedman1} & \textbf{Diabetes} & \textbf{Friedman 2} &  \textbf{Large}  \\\midrule
\textbf{Our Lasso} &\textbf{1.706}&\textbf{1.945} &\textbf{1.0785} &\textbf{0.681} & \textbf{150.012}\\
\textbf{Appr. Lasso} &5.176 &43.823 &70.813 &14.055 & 500.820\\
\textbf{Our Robust}&\textbf{27.156}&1.069 &2.411 &0.701& \textbf{323.372}\\
\textbf{Appr. Robust} &62.894&1.009 &2.734 &0.618 & 607.203\\
\textbf{Our Asym.} &\textbf{9.270} &3.147 &\textbf{27.349} &2.454 & \textbf{41.269} \\
\textbf{Appr. Asym.} &18.963&2.699 &54.149 &3.342 & 82.857\\
\bottomrule

\end{tabular}
\end{table}

 \begin{table}[!htp]\centering
\scriptsize
\caption{Average Negative Logarithm of Primal Gap of Homotopy }
\begin{tabular}{@{}lccccl@{}}\toprule
\label{tab:homotopy_primal-diff}
& \multicolumn{4}{c}{\textbf{Dataset}} \\
\cmidrule(lr){2-6}
&\textbf{Synth.}&\textbf{Friedman1} & \textbf{Diabetes} & \textbf{Friedman2} & \textbf{Large}\\
\midrule
\textbf{Our Lasso}  &\textbf{12.498}&\textbf{15.844} &\textbf{16.001} &\textbf{15.241}& 7.933  \\
\textbf{Appr. Lasso} &6.597&6.469 &7.554 &6.702 & 7.558\\
\textbf{Our Robust} &5.137&2.317 &3.819 &2.778 & 5.223 \\
\textbf{Appr. Robust} &5.990&3.561 &3.712 &4.434 & 5.026 \\
\textbf{Our Asym.} &7.879&3.633 &3.814 &3.058 & 6.101\\
\textbf{Appr. Asym.} &6.939&3.208 &4.032 &2.795 & 5.365 \\
\bottomrule
\end{tabular}
\end{table}

\subsection{Conformal Prediction Experiments}
\begin{figure*}[!ht]
        \centering
    \begin{minipage}[t]{\linewidth}   
            \centering 
                \begin{table}[H]\centering

\scriptsize
\begin{tabular}{lrrrrrrrrrrrrr}\toprule
&\multicolumn{3}{c}{Diabetes Coverage} &\multicolumn{3}{c}{Friedman 1 Coverage} &\multicolumn{3}{c}{Friedman 2 Coverage} &\multicolumn{3}{c}{Synthetic Coverage} \\\cmidrule(lr){2-4}\cmidrule(lr){5-7}\cmidrule(lr){8-10}\cmidrule(lr){11-13}
&\textbf{Lasso} &\textbf{Robust} &\textbf{Asymmetric} &\textbf{Lasso} &\textbf{Robust} &\textbf{Asymmetric} &\textbf{Lasso} &\textbf{Robust} &\textbf{Asymmetric} &\textbf{Lasso} &\textbf{Robust} &\textbf{Asymmetric} \\
\textbf{Ours} &0.933 &0.933 &0.867 &0.900 &0.900 &0.883 &0.900 &0.850 &0.933 &0.900 &0.900 &0.900 \\
\textbf{Approximate} &0.933 &0.933 &0.867 &0.850 &0.900 &0.833 &0.900 &0.800 &0.900 &0.850 &0.900 &0.850 \\
\textbf{Split} &0.933 &0.867 &0.800 &0.867 &0.933 &0.867 &1.000 &0.867 &0.933 &1.000 &1.000 &0.900 \\
\textbf{Grid} &0.933 &0.933 &0.867 &0.767 &0.767 &0.767 &0.900 &0.767 &0.933 &0.850 &0.850 &0.850 \\
\textbf{Oracle} &0.933 &0.933 &0.867 &0.867 &0.967 &0.933 &0.900 &0.867 &0.933 &1.000 &0.900 &1.000 \\
\bottomrule
\end{tabular}
\caption{Coverage Results over Several Datasets}\label{tab:coverage}
\end{table}

\label{fig:asyn_cp}
\end{minipage}
\begin{minipage}[t]{\linewidth}
       \begin{table}[H]\centering

\scriptsize
\begin{tabular}{lrrrrrrrrrrrrr}\toprule
&\multicolumn{3}{c}{Diabetes Length} &\multicolumn{3}{c}{Friedman 1 Length} &\multicolumn{3}{c}{Friedman 2 Length} &\multicolumn{3}{c}{Synthetic Length} \\\cmidrule(lr){2-4}\cmidrule(lr){5-7}\cmidrule(lr){8-10}\cmidrule(lr){11-13}
 &\textbf{Lasso} &\textbf{Robust} &\textbf{Asymmetric} &\textbf{Lasso} &\textbf{Robust} &\textbf{Asymmetric} &\textbf{Lasso} &\textbf{Robust} &\textbf{Asymmetric} &\textbf{Lasso} &\textbf{Robust} &\textbf{Asymmetric} \\
\textbf{Ours}&2.234 &2.230  &2.356 &2.340 &2.570 &2.875 &2.004 &2.191 &2.621 &0.632 &0.714 &0.702 \\
\textbf{Approximate} &2.262 &2.237 &2.381  &2.368 &2.599 &2.897 &2.024 &2.245 &2.618 &0.786 &0.705 &0.790 \\
\textbf{Split}  &2.409 &2.429 &2.469 &2.589 &2.837 &3.219 &2.361 &2.448 &2.888 &0.831 &0.831 &0.831 \\
\textbf{Grid}  &2.286 &2.255 &2.390 &2.475 &2.782 &2.982 &2.108 &2.338 &2.741 &0.872 &0.903 &0.651 \\
\textbf{Oracle} &2.320 &2.337 &2.396  &2.204 &2.508 &2.787 &2.240 &2.447 &2.550 &0.062 &0.001 &0.226 \\
\bottomrule
\end{tabular}
\caption{Length Results over Several Datasets}\label{tab:length}
\end{table}
\end{minipage}
\begin{minipage}[t]{\linewidth}
\begin{table}[H]\centering
\scriptsize
\begin{tabular}{lrrrrrrrrrrrrr}\toprule
&\multicolumn{3}{c}{Diabetes Time (s)} &\multicolumn{3}{c}{Friedman 1 Time (s)} &\multicolumn{3}{c}{Friedman 2 Time (s)} &\multicolumn{3}{c}{Synthetic Time (s)} \\\cmidrule(lr){2-4}\cmidrule(lr){5-7}\cmidrule(lr){8-10}\cmidrule(lr){11-13}
&\textbf{Lasso} &\textbf{Robust} &\textbf{Asymmetric} &\textbf{Lasso} &\textbf{Robust} &\textbf{Asymmetric} &\textbf{Lasso} &\textbf{Robust} &\textbf{Asymmetric} &\textbf{Lasso} &\textbf{Robust} &\textbf{Asymmetric} \\
\textbf{Ours} &0.658 &1.125 &6.865 &0.493 &0.398 &1.047 &0.326 &0.305 &0.901 &6.056 &115.565 &4.037 \\
\textbf{Approximate} &4.012 &33.858 &213.995 &5.127 &4.234 &17.111 &1.538 &2.991 &13.178 &30.124 &332.152 &9.142 \\
\textbf{Split} &0.025 &0.086 &0.474 &0.139 &0.041 &0.102 &0.036 &0.042 &0.100 &0.599 &0.122 &0.039 \\
\textbf{Grid} &0.769 &5.692 &58.632 &1.704 &2.238 &11.331 &0.564 &2.068 &8.834 &29.150 &431.398 &2.418 \\
\textbf{Oracle} &0.049 &0.188 &1.032 &0.116 &0.034 &0.212 &0.040 &0.033 &0.193 &0.429 &0.256 &0.063 \\
\bottomrule
\end{tabular}
\caption{Time Results over Several Datasets}\label{tab:time}
\end{table}
\end{minipage}
\caption{We demonstrate the performance of our Conformal Set Algorithm against several baselines across many datasets and loss functions. Our algorithm maintains strong coverage, length, and time metrics across many loss functions and datasets.}
\end{figure*}
It is a natural question whether this improvement in the generation of the homotopy function yields a strong conformal set generation algorithm. We demonstrate this both visually and empirically. We draw the $\pi(z)$ function for visual verification over all four datasets and three loss functions using our algorithm. To form a baseline, we use the Grid algorithm. This algorithm is a ground truth to which we compare our $\pi(z)$ function. For empirical verification, we compare coverage, length, and the time of our method vs. several important baselines. Namely, we use the Grid method, Approximate homotopy from \cite{Ndiaye_Takeuchi19}, the Oracle methodology, which has access to the true value of $y_{n+1}$ to form its conformal interval, and the Split methodology, which uses a calibration dataset to calibrate the conformal values predicted but loses statistical validity.

\paragraph{Visual Results} We report the figures in \Cref{fig:all_cp}. As is evident over all loss functions and datasets, our estimated $\pi(z)$ roughly traces the true $\pi(z)$ generated by the discretized searching algorithm. While on particular examples, notably Figures \ref{fig:rsyn_cp},  \ref{fig:af1_cp}, and \ref{fig:ad_cp}, the trace is less accurate than the others. However, the error is within a reasonable range to achieve the desired coverage and length guarantees. We also report similar experiments for the Lasso loss, but we mention these in \Cref{sec:additionalvisual} since our method is exact for the Lasso loss. We demonstrate that our homotopy drawing algorithm yields an efficient and accurate methodology for generating conformal sets for general loss functions as tested on several datasets. \looseness=-1

\paragraph{Empirical Results} We report our empirical results in \Cref{tab:coverage}, \Cref{tab:length}, and \Cref{tab:time}. We can see that most methods maintain strong coverage guarantees over all datasets. For our experiments, we used $\alpha = 0.1$, and most of our results hover around that level of coverage. Moreover, in \Cref{tab:length}, we see that except for Oracle, across several loss functions and datasets, our algorithm achieves the smallest length. The Oracle, however, consistently has the best length due to its knowledge of the true $y_{n+1}$. Also, our algorithm is the fastest over all homotopy methods but slower than Split and Oracle, as seen in \Cref{tab:time}. Therefore, our experiments indicate that our Conformal Prediction Algorithm is competitive in all coverage, length, and time measures. \looseness=-1

\section{Conclusion}
Our results demonstrate that we can efficiently and accurately draw the homotopy of the typicalness function of a model over several loss functions via exploiting the sparsity structure of the Linear Models with $\ell_1$ regularization. Furthermore, we achieve explicit closed-form equations to model the behavior of this homotopy. Previous results mainly focus on quadratic loss functions or ignore the structure of the regularization altogether. Our framework, instead, captures this information and uses it to improve the accuracy of our final results. 
Several avenues for extending our research remain interesting. Spline instead of linear interpolation may yield improved accuracy for different loss functions. Additionally, smoothing at the kinks may reduce the algorithm's sensitivity to the primal corrector's results. Furthermore, we would like to expand our work to non-convex settings such as deep learning in future works.

\bibliographystyle{icml2023}
\bibliography{ref}
\appendix
\newpage
\onecolumn
\appendix

\section{Additional Visualizations}
We have provided two extra visualizations for the reader's understanding. We have provided figures of the homotopy for different loss functions and what the conformality function $\pi(z)$ looks like for the Quadratic Loss function on both our real and synthetic datasets. 
\label{sec:additionalvisual}
\paragraph{Homotopy Visualizations}
We run our homotopy generation algorithm over Quadratic Loss, Robust, and Asymmetric Loss functions over several low-dimensional synthetic examples. As we can see in \Cref{fig:lf1_cp}, for the Quadratic Loss, our homotopy algorithm perfectly matches that of the Grid baseline since our algorithm captures the Quadratic Loss homotopy from \citet{Lei19} exactly. Moreover, in \Cref{fig:lf2_cp} and \Cref{fig:ld_cp}, we see that our algorithm very closely identifies the homotopy of the Grid algorithm. In the Robust and Asymmetric cases, the linearization causes a slight miss in the kink, but the difference is negligible. Across all dimensions, our homotopy generation algorithm closely tracks that of the baseline Grid algorithm. These visualizations verify visually that our homotopy generation algorithm is accurate.

\paragraph{Conformity Function for Quadratic Loss} We also visualize what the conformality function $\pi(z)$ looks like across several datasets for the Quadratic Loss function. We do not include these in the main manuscript since our algorithm is exact on the Quadratic Loss function, and no visual verification is truly needed. Nevertheless, we provide such visuals in \Cref{fig:mean and std of nets}. Our Conformal Prediction algorithm indeed matches precisely that of the Grid baseline algorithm. This confirms our claims that our algorithm is indeed exact on the Quadratic Loss function. 

\begin{figure*}
        \centering
        \begin{subfigure}[b]{0.30\textwidth}
            \centering
            \includegraphics[width=\textwidth]{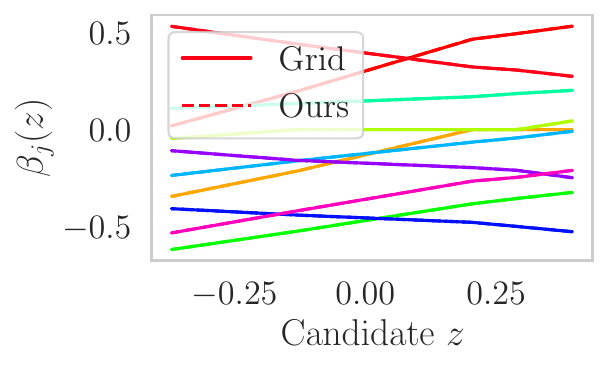}
            \caption[Network2]%
            {{\small Lasso}}    
            \label{fig:lf1_cp}
        \end{subfigure}
        \hfill
        \begin{subfigure}[b]{0.30\textwidth}  
            \centering 
            \includegraphics[width=\textwidth]{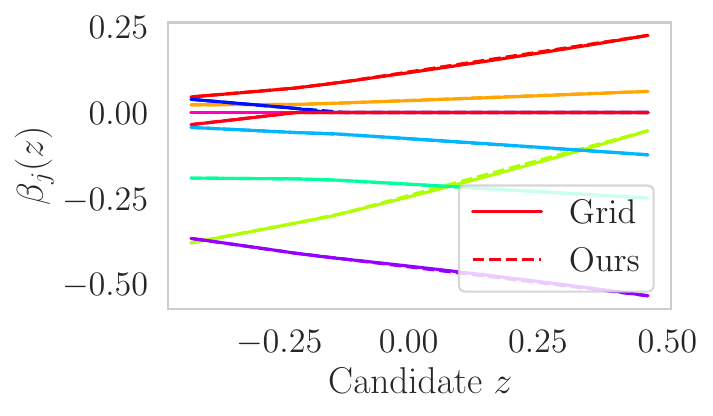}
            \caption[]%
            {{\small Asymmetric}}    
            \label{fig:lf2_cp}
        \end{subfigure}
        \hfill
        \begin{subfigure}[b]{0.30\textwidth}   
            \centering 
            \includegraphics[width=\textwidth]{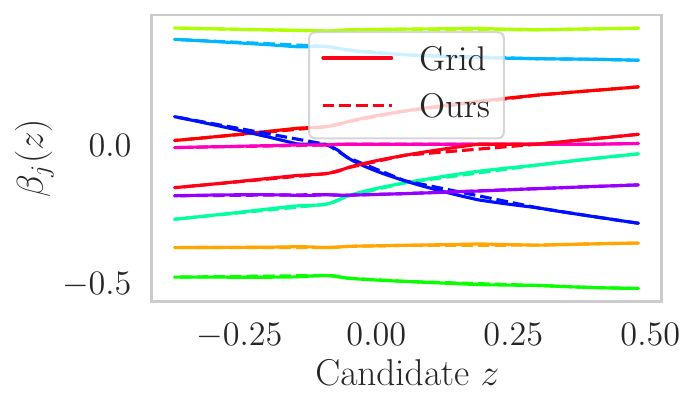}
            \caption[]%
            {{\small Robust}}    
            \label{fig:ld_cp}
        \end{subfigure}
        \label{fig:figure1}
        \caption{We generate several example homotopies over Lasso, Asymmetric, and Robust loss functions. We plot using our algorithm and a discretized search space algorithm, where the space of potential $z_{n+1}$ values is split into several points, and we solve for $\beta$ using Proximal Gradient Descent at each point.}
\end{figure*}
\begin{figure*}
    \begin{subfigure}[b]{0.240\textwidth}
            \centering
            \includegraphics[width=\textwidth, scale=.25]{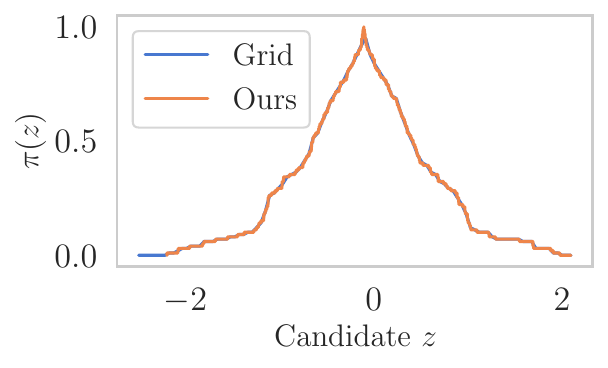}
            \caption[Network2]%
            {{\small Lasso on Friedman1}}    
            \label{fig:lf1_cp}
        \end{subfigure}
        \hfill
        \begin{subfigure}[b]{0.240\textwidth}  
            \centering 
            \includegraphics[width=\textwidth]{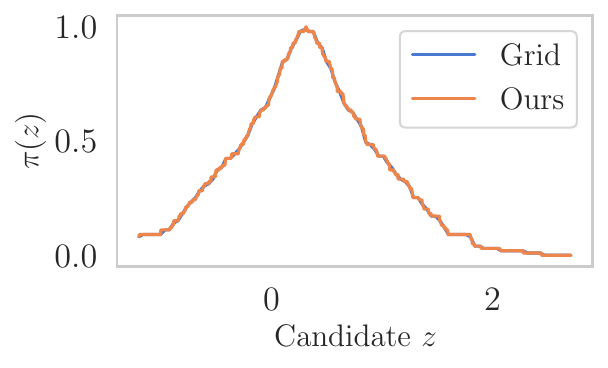}
            \caption[]%
            {{\small Lasso on Friedman 2}}    
            \label{fig:lf2_cp}
        \end{subfigure}
        \begin{subfigure}[b]{0.240\textwidth}   
            \centering 
            \includegraphics[width=\textwidth]{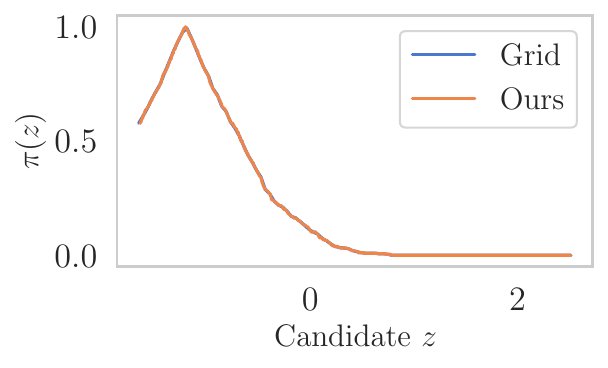}
            \caption[]%
            {{\small Lasso on Diabetes}}    
            \label{fig:ld_cp}
        \end{subfigure}
        \hfill
        \begin{subfigure}[b]{0.240\textwidth}   
            \centering 
            \includegraphics[width=\textwidth]{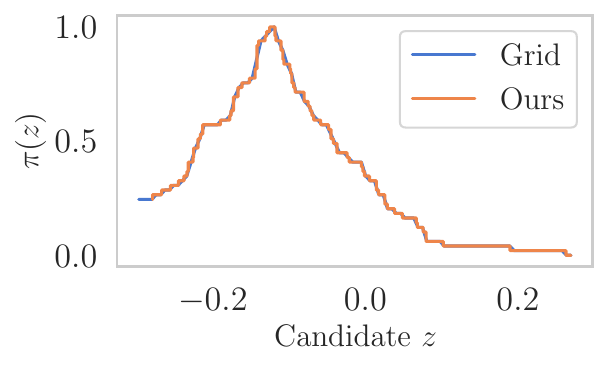}
            \caption[]%
            {{\small Lasso on Synthetic}}    
            \label{fig:lsyn_cp}
        \end{subfigure}
        \caption[ The average and standard deviation of critical parameters ]
        {\small The $\pi(z)$ function as generated by a ground truth discretized searching algorithm and by our homotopy drawing algorithm for the Quadratic Loss function over all $4$ datasets.} 
        \label{fig:mean and std of nets}
\end{figure*}

\newpage
\section{Proofs for Properties of GLM's}
\subsection{Proof of \Cref{lem:uniqueness_solution}}
\uniquesolution*

\begin{proof}
The Fermat rule reads 
$$0 \in \{X^{\top} \partial_2 f(y(z),y^{\star}(z))\} + \lambda \partial \norm{\cdot}_1(\beta^{\star}(z)) \enspace.$$
Defining $v(z) \in \partial \norm{\cdot}_1(\beta^{\star}(z))$ yields \Cref{eq:fermat_equality}. To show \Cref{eq:subdifferential_l1norm}, we look at $v_j(z)$ for any index $j$.
We remind that by separability of the $\ell_1$ norm, we have
$v_j(z) = \partial |\cdot|(\beta_j^{\star}(z))$. Hence, $v_j(z) = \operatorname{sign}(\beta_j^{\star}(z))$ if $\beta_j^{\star}(z) \neq 0$ and $v_j(z) \in [-1, 1]$ otherwise. This proves the claim.

\end{proof}
\subsection{Proof of \Cref{lem:unique_path}}
\uniquepath*
\begin{proof}
We first prove that $A(z)$ is unique. From the definition of the active set, we have $$A(z) = \left\{j \in [p]:\; |X_{j}^{\top}\partial_2 f(y(z),X\beta^{\star}(z))| = \lambda \right\}\enspace,$$
where we remind that 
$$\beta^{\star}(z) \in \argmin_{\beta \in \bbR^{|p|}} f(y(z),X \beta) + \lambda \norm{\beta}_1 \enspace.$$
Since, from strict convexity of $f$, the prediction $X\beta^{\star}(z)$ is unique for any solution $\beta^{\star}(z)$ to the aforementioned optimization problem, we have $A(z)$ is uniquely defined. From the first order optimality condition, it exists $v(z) \in \partial \norm{\cdot}_1(\beta^\star(z))$
$$
0 \in X^\top \partial_2 f(y(z), X\beta^\star(z)) + \lambda v(z) \enspace.
$$
Restricted to the active set yields
$$
0 \in X_{A}^\top \partial_2 f(y(z), X_{A}\beta_{A}^\star(z)) + \lambda v_{A}(z)
\Longleftrightarrow
\beta_A^{\star}(z) \in \argmin_{w \in \bbR^{|A|}} f(y(z),X_A w) + \lambda \norm{w}_1 \enspace.$$
Since $f$ is strictly convex and $X_A$ is full rank, the latter optimization problem is strictly convex meaning $\beta_A^{\star}(z)$ is unique. 
\end{proof}

\subsection{Proof of \Cref{lem:compact}}
\begin{restatable}[]{lem}{compact}
\label{lem:compact}
Let $\mathbf{0}$ be the vector of $0$'s, For all $z \in [y_{\min}, y_{\max}]$, we have that the optimal weights $\beta^{\star}(z)$ satisfy $$\{ \beta^{\star}(z) : z \in [z_{\min}, z_{\max}] \} \subset \{ \beta : \norm{\beta}_1 \leq {R}/{\lambda} \} \text{ where } R = \sup_{z \in [z_{\min}, z_{\max}]} f(y(z), \mathbf{0})\text{.}$$

\end{restatable}
\begin{proof}
Let's denote the objective function as 
$$ P(\beta, z) = f(y(z), X\beta) + \lambda \norm{\beta}_1 \text{.}$$

We remind the reader that the solution $\beta^{\star}(z)$ satisfies $\beta^{\star}(z) = \argmin_{\beta} P(\beta, z)$. By optimality and assuming that $f$ is non-negative, we have for any $z$
\begin{align*}
\lambda \norm{\beta^{\star}(z)}_1 &\leq P(\beta^{\star}(z), z)
\leq P(\mathbf{0}, z) = f(y(z), \mathbf{0}) \enspace.
\end{align*} 
Here, $\mathbf{0}$ is the vector of $0$'s, the first step comes from the definition of $P$, and the second inequality comes from the fact that $\beta^{\star}(z)$ is a minimizer of $P$.
Naturally, we then have that the $\ell_1$ norm of the the solving weights $\beta^{\star}(z)$ is bounded by the value of $f(y(z), \mathbf{0})$. Any solution $\beta^{\star}(z)$ is inside the $\ell_1$ ball centered at $\mathbf{0}$ with radius $R/\lambda$. 

Since the path is truncated \ie $z \in [z_{\min}, z_{\max}]$, then the solution path is bounded \ie
$$ \{ \beta^{\star}(z) : z \in [z_{\min}, z_{\max}] \} \subset \left\{ \beta : \norm{\beta}_1 \leq \frac{R}{\lambda} \right\}  \enspace,$$

where  
$$ R= \sup_{z \in [z_{\min}, z_{\max}]} f(y(z), \mathbf{0}) \enspace.$$
\end{proof}

Also, it is easy to see that, along the path
\begin{align*}
\norm{y(z)}_2 &\leq \max(\norm{y(z_{\min})}_2, \norm{y(z_{\max})}_2) \\
\norm{y^\star(z)}_2 &= \norm{X_A \beta_A^\star}_2 \leq  \frac{\sigma_{\max}(X_A) \times R}{\lambda} \\
\norm{\zeta(z)}_2 &= \sqrt{\|y(z)\|_2^2 + \|y^*(z)\|_2^2} \\
&\leq \sqrt{ \max(\norm{y(z_{\min})}_2, \norm{y(z_{\max})}_2)^2 + \left(\frac{\sigma_{\max}(X_A) \times R}{\lambda}\right)^2 } =: B \enspace.
\end{align*}
Note that, for simplicity, we naturally suppose that the estimate $\hat \beta(z)$ is a better minimizer than the vector $\mathbf{0}$. Thus, the same bounds above hold for $\hat\beta(z)$, $\hat y(z)$ and $\hat \zeta(z)$.

\section{Choice of the Range $[z_{\min}, z_{\max}]$}
\begin{restatable}[]{lem}{justified}
\label{lem:justified}
Choosing $z_0 = z_{\max} = \max(y_1, \ldots, y_n)$ and stopping once $z_t \leq z_{\min} = \min(y_1, \ldots y_n)$ reduces the probability of coverage by at most $\frac{2}{n+1}$.
\end{restatable}
\begin{proof}
    Given our exchangability assumption, the probability that $y_{n+1} \geq z_{\max}$ is at most $\frac{1}{n+1}$. Similarly, the probability that $y_{n+1} \leq z_{\min}$ is at most $\frac{1}{n+1}$.  Therefore, using the union bound, the probability that choosing the criteria we do in our algorithm affects coverage by at most $\frac{2}{n+1}$, which becomes negligible as $n$ grows.
\end{proof}
\section{Details on $\partial_2 f$}
\label{sec:detailspartf}
In the main text, we mentioned that we are approximating $\partial_2 f(y(z^+), y^\star(z^+))$. We can do this via linearization.

$$
\partial_2 f \circ \zeta(z^+) \approx \partial_2 f \circ \zeta(z) + \partial_{2, 1} f \circ \zeta(z) (y(z^+) - y(z)) + \partial_{2, 2} f \circ \zeta(z)(y^\star(z^+) - y^\star(z)) 
$$
Moreover,
\begin{align*}
y(z^+) - y(z) &= (0, \ldots, 0, z^+ - z) \\
y^\star(z^+) - y^\star(z) &= X_A (\beta_A^\star(z^+) - \beta_A^\star(z)) \approx X_A \frac{\partial \beta_A^\star}{\partial z}(z) \times (z^+ - z)
\end{align*}

Finally, with a plug-in approach, we approximate $\hat\beta_A^\prime(z) \approx \frac{\partial \beta_A^\star}{\partial z}(z)$ and obtain

$$
\partial_2 f \circ \zeta(z^+) \approx \partial_2 f \circ \zeta(z) + \bigg( [\partial_{2,1} f \circ \zeta(z)]_{n+1} + \left[\partial_{2, 2} f \circ \zeta(z)\right] X_A \hat\beta_A^\prime(z) \bigg) \times (z^+ - z)
$$

Here, the first equality come from definition, the second is from applying the chain rule to intermediate variables, the third is from a simple notational switch, and the final equality comes from using our estimators for $\beta^\star$. Now, we can find the roots of $\mathcal{I}$ from \Cref{eq:invariable} as the following
$$z_{j, t+1}^{\rm{in}} = z_t +  \frac{-X_{j}^{\top} \partial_2 f \circ \hat{\zeta}(z_t) \pm \lambda}{X_{j}^{\top}\left[[\partial_{2,1} f\circ\hat\zeta(z_t)]_{n+1} + \partial_{2,2}f\circ\hat\zeta(z_t)^\top X_A \hat\beta_A^\prime(z_t)\right]} \enspace.$$

We can now present our desired theorem.

\begin{restatable}{lem}{boundgradient}\label{lm:bound_gradient}
    The gradient of the solution path $\frac{\partial \beta^\star}{\partial z}(z)$, as well as its estimates $\hat\beta^\prime(z)$ are bounded as follow
    $$
\max\left( \norm{\frac{\partial \beta^\star}{\partial z}(z)}, \norm{\hat\beta^\prime(z)}  \right) \leq \frac{\sigma_{\max}(X_{A(z)})}{\sigma_{\min}^2(X_{A(z)})} \times \frac{\nu_f}{\mu_f} \enspace.
    $$
\end{restatable}

\begin{proof}
We remind that $\frac{\partial y}{\partial z} = (0, \ldots, 0, 1)^{\top}$ and
\begin{align*}
    &\frac{\partial \beta_{A}^{\star}}{\partial z} = -\left(\frac{\partial H}{\partial \beta}\right)^{-1}\frac{\partial H}{\partial y}\frac{\partial y}{\partial z} 
    & \hat\beta_{A}^{\prime} = -\left(\widehat{\frac{\partial H}{\partial \beta}}\right)^{-1}\widehat{\frac{\partial H}{\partial y}}\frac{\partial y}{\partial z} \\
    &\frac{\partial H}{\partial \beta} = X_A^\top\partial_{2, 2} f \circ \zeta(z)X_A 
    &\widehat{\frac{\partial H}{\partial \beta}} = X_A^\top\partial_{2, 2} f\circ \hat \zeta(z)X_A \\
   &\frac{\partial H}{\partial y} = X_A^\top\partial_{2,1} f \circ \zeta(z)
   &\widehat{\frac{\partial H}{\partial y}} = X_A^\top\partial_{2,1} f\circ \hat \zeta(z)
\end{align*}
Hence
$$
\norm{\frac{\partial \beta_{A}^{\star}}{\partial z}}_2 \leq \norm{\left(\frac{\partial H}{\partial \beta}\right)^{-1}}_2\norm{\frac{\partial H}{\partial y}}_2
$$

By definition, we have that for any $z \in [z_{\min}, z_{\max}]$, 
   $$\norm{\frac{\partial H}{\partial \beta} (y(z), \beta_A^{\star}(z))}_2 = \norm{X_A^\top\partial_{2, 2} f\circ\zeta(z) X_A}_2 \geq \sigma_{\min}(X_A^\top X_A) \times  \inf_{\norm{\zeta} \leq B} \sigma_{\min}(\partial_{2, 2}f \circ \zeta(z) )\enspace. $$
Since $f$ is assumed to be $\mu_f$-strongly convex from \Cref{lm:strong_convexity_along_path}, it holds
$$\inf_{\norm{\zeta} \leq B} \sigma_{\min}(\partial_{2, 2}f \circ \zeta(z) ) \geq \mu_f > 0 \enspace,$$
and then
\begin{align*}
\norm{\left(\frac{\partial H}{\partial \beta}\right)^{-1}}_2 &\leq \frac{1}{\sigma_{\min}^2(X_A) \times \mu_f} \enspace.
\end{align*}

Similarly, given $f$ is smooth with constant $\nu_f$ from \Cref{lm:bound_derivatives}, we have 
\begin{align*}
\norm{\frac{\partial H}{\partial y} (y(z), \beta_A^{\star}(z)) }_2 &= \norm{X_A^\top\partial_{2,1} f \circ \zeta(z) }_2 
\leq \sigma_{\max}(X_A) \norm{\partial_{2,1} f \circ \zeta(z) }_2 \leq \sigma_{\max}(X_A) \times \nu_f \enspace.
\end{align*}
Hence the result.
The proof for upper-bounding the estimated gradient norm follows the same line.
\end{proof}

\wholeerror*
\begin{proof}
To analyze $\|\tilde{\beta}(z^{+}) - \beta^\star(z^{+})\|_2$, we will use the definition from our algorithm that 
    
    $$\tilde{\beta}(z^{+}) = \hat \beta(z_t) + \hat \beta^\prime(z_t) (z^{+} - z_t)\text{.}$$ 
    Here, $z_t$ is the last point at which we ran our primal corrector. 
    Using this, we can decompose the error as the follows:
    \begin{align*}
        \left\|\tilde{\beta}(z^{+}) - \beta^\star(z^{+})\right\|_2 &= \left\|\hat\beta(z_t) + \hat \beta^\prime(z_t) (z^{+} - z_t) - \beta^\star(z^{+}) \right\|_2 \\
        &= \left\|\hat\beta(z_t) + \hat \beta^\prime(z_t) (z^{+} - z_t) - \beta^\star(z^{+}) + \beta^\star(z_t) - \beta^\star(z_t) \right\|_2 \\
        &= \left\|\hat\beta(z_t) - \beta^\star(z_t) +  \int_{z_t}^{z^{+}} \left[\hat \beta^\prime(z_t) - \frac{\partial \beta^{\star}(z)}{\partial z}\right] dz \right\|_2 \\
        &\leq \norm{\hat\beta(z_t) - \beta^\star(z_t)}_2 + \sup_{z \in [z^+, z_t]} \norm{ \hat \beta^\prime(z_t) - \frac{\partial \beta^{\star}(z)}{\partial z} }_2 |z^+ - z_t| \\
        %
        %
    \end{align*}
Here, the third equality comes from the fact that $\beta^*(z^+)- \beta^*(z_t) = \int_{z_t}^{z^{+}} \frac{\partial \beta^{\star}(z)}{\partial z} dz $. 

Now, from the Triangular Inequality, we have
\begin{align*}
\sup_{z \in [z^+, z_t]} \norm{ \hat \beta^\prime(z_t) - \frac{\partial \beta^{\star}(z)}{\partial z} }_2 &\leq  \normin{\hat \beta^\prime(z_t)} + \sup_{z \in [z^+, z_t]} \norm{\frac{\partial \beta^{\star}(z)}{\partial z} }_2 \\
&\leq \left[ \frac{\sigma_{\max}(X_{A(z_t)})}{\sigma_{\min}^2(X_{A(z_t)})} + 
\sup_{z \in [z^+, z_t]} \frac{\sigma_{\max}(X_{A(z)})}{\sigma_{\min}^2(X_{A(z)})} \right] \times \frac{\nu_f}{\mu_f} 
\end{align*}
Here, the second inequality comes from \Cref{lm:bound_gradient}.

Now, if point $z^+$ is such that $z^+ \in (z_{t+1}, z_t]$, the active sets at point $z^+$ and $z_t$ are constant. Then, we can simplify. 
$$
\norm{\tilde{\beta}(z^{+}) - \beta^\star(z^{+})}_2 \leq \epsilon_{\rm{tol}} +  \frac{2 \,\sigma_{\max}(X_{A(z_t)})}{\sigma_{\min}^2(X_{A(z_t)})} \times \frac{\nu_f}{\mu_f} \times |z^+ - z_t| \enspace.
$$

Note that if the candidate $z^+ = z_t$ is exactly a kink, the right-most term is zero. It only remains the corrector error. If it is not the case that $z^+$ and $z_t$ have the same active set, we have 
$$
\norm{\tilde{\beta}(z^{+}) - \beta^\star(z^{+})}_2 \leq \epsilon_{\rm{tol}} +  L \times \frac{\nu_f}{\mu_f} \times |z^+ - z_t| \enspace
$$
where
\begin{equation}
L :=
\left[\frac{\sigma_{\max}(X_{A(z_t)})}{\sigma_{\min}^2(X_{A(z_t)})} + 
\sup_{z \in [z^+, z_t]} \frac{\sigma_{\max}(X_{A(z)})}{\sigma_{\min}^2(X_{A(z)})} \right]
\end{equation}
\end{proof}

\uppboundg*

\begin{proof}
Let us define, for $t \in [0, 1]$, the function
$$ \phi(t) = \partial_2 f(\hat \zeta(z^+) + t(\zeta(z^+) -  \hat \zeta(z^+))) \enspace.$$
We have from the fundamental theorem of calculus, $\phi(1) - \phi(0) = \int_{0}^{1} \frac{\partial \phi(t)}{\partial t} dt$ where
$$\phi(1) - \phi(0) = \partial_2 f(\zeta(z^+)) - \partial_2 f(\hat \zeta(z^+))$$
and
$$\frac{\partial \phi(t)}{\partial t} = \partial_{2,1} f(\hat \zeta(z^+) + t(\zeta(z^+) -  \hat \zeta(z^+)))^\top[\zeta(z^+) -  \hat \zeta(z^+)]_1 + 
\partial_{2,2} f(\hat \zeta(z^+) + t(\zeta(z^+) -  \hat \zeta(z^+)))^\top[\zeta(z^+) -  \hat \zeta(z^+)]_2 \enspace.$$
We remind the reader that, by definition, we have 
\begin{align*}
    [\zeta(z^+) -  \hat \zeta(z^+)]_1 &= y(z^+) - y(z^+) = \textbf{0} \\
    [\zeta(z^+) -  \hat \zeta(z^+)]_2 &= y^\star(z^+) - \hat y(z^+)
\end{align*}
and deduce that
\begin{align*}
\normin{\partial_2 f \circ \zeta(z^+) - \partial_2 f \circ \hat \zeta(z^+)}_2 &= \|\phi(1) - \phi(0)\|_2\\
&= \left\|\int_{0}^{1} \frac{\partial \phi(t)}{\partial t} dt\right\|_2\\
&=\left\| \int_{0}^{1} \partial_{2,2} f(\hat \zeta(z^+) + t(\zeta(z^+) -  \hat \zeta(z^+)))^\top[\zeta(z^+) -  \hat \zeta(z^+)]_2 dt \right\|_2\\
&\leq \sup_{t \in [0, 1]} \normin{\partial_{2,2} f(\hat \zeta(z^+) + t(\zeta(z^+) -  \hat \zeta(z^+)))}_2 \normin{y^\star(z^+) - \hat y(z^+)}_2\\
&\leq \nu_f \times \normin{X_A \beta_A^\star(z^+) -  X_A \hat \beta_A(z^+)}_2 \\
&\leq \nu_f \times \sigma_{\max}(X_A) \times \normin{ \beta_A^\star(z^+) - \hat \beta_A(z^+)}_2 \\
&\leq \nu_f \times \sigma_{\max}(X_A) \times \left[ \epsilon_{\rm{tol}} +  L \times \frac{\nu_f}{\mu_f} \times |z^+ - z_t| \right]
\enspace.
\end{align*}
Here, the fourth inequality comes from our \Cref{lm:bound_derivatives} and the final inequality comes from \Cref{thm:wholeerror}.
\end{proof}

\newpage

\section{Partial Linearization as Alternative Methods}

In addition to the approximation algorithm described and analyzed above, we briefly describe a more precise but more costly method in terms of computational time. The key point is to try to capture the non-linearity of the solution-path as the input z varies or similarly when the regularization parameter $\lambda$ varies.

For the sake of simplicity, we describe a solution path that exploits the exact solution at each node. The practical algorithm will be based on a plug-in approach similar to that used above.

In the following, we note (deleting the star notation to avoid clutter)

\begin{equation*}
\beta(z, \lambda) \in \argmin_{\beta \in \mathbb{R}^p} f(y(z),X\beta) + \lambda \norm{\beta}_1 \nonumber \enspace.
\end{equation*}

Let us define $f_z(q) = f(y(z), q)$ and linearly approximate the function $q \mapsto \nabla f_z(q)$ at $q_0$ \ie

\begin{equation}\label{eq:linearizing_nablaf}
 \nabla f_z(q) \approx  \nabla f_z(q_0) +  \nabla^2 f_z(q_0)^\top (q - q_0) 
\end{equation}

We denote $q = X_A\beta_A(z, \lambda)$ and $q_0 = X_A\beta_A(z_0, \lambda)$. From the optimality condition \cref{eq:fermat_equality}, we have

\comment{We only linearize \wrt second variable of $\partial_2 f$ which is independent of $z$}

\begin{align}
    X_{A}^{\top}\nabla f_z(q) &= - \lambda v_A(z, \lambda) \nonumber\\
    X_{A}^{\top} \bigg( \nabla f_z(q_0) +  \nabla^2 f_z(q_0)^\top (q - q_0) \bigg) &\overset{\eqref{eq:linearizing_nablaf}}{\approx} - \lambda v_A(z, \lambda) \nonumber\\
    X_{A}^{\top} \nabla^2 f_z(q_0)^\top q &\approx X_{A}^{\top} \bigg(\nabla^2 f_z(q_0)^\top q_0 - \nabla f_z(q_0) \bigg) - \lambda v_A(z, \lambda)\nonumber
\end{align}

Hence
\begin{equation}\label{eq:complex_primal}
    \beta_A(z, \lambda) \approx  \bigg(X_{A}^{\top} \nabla^2 f_z(q_0)^\top X_A\bigg)^{-1} \bigg( X_{A}^{\top} \big(\nabla^2 f_z(q_0)^\top q_0 - \nabla f_z(q_0) \big) - \lambda v_A(z, \lambda) \bigg) 
\end{equation}

We recover exactly the Lasso formula when $f$ is quadratic and also we only need to know $v$ and not the dual variable.

\paragraph{Path \wrt $\lambda$ ($z$ is fixed).}

We have the two situations for a change in the active set:

\begin{itemize}
\item A nonzero variable becomes zero \ie

$$\exists j \in A(z, \lambda) \text{ such that }: \beta_j(z, \lambda) \neq 0 \text{ and } \beta_j(z, \lambda_{j, \rm{out}}) = 0 \enspace.$$

\item A zero variable becomes nonzero \ie

$$ \exists j \in A^c(z, \lambda) \text{ such that }: |X_{j}^{\top}\nabla f_z(X\beta(z, \lambda_{j, \rm{in}}))| = \lambda_{j, \rm{in}} \enspace.$$
\end{itemize}

Then 
\begin{equation}\label{eq:next_lambda}
    \lambda_{\rm{next}} = \max\left(\max_{j \in A(z, \lambda)} \lambda_{j, \rm{out}},\; \max_{j \in A^c(z, \lambda)} \lambda_{j, \rm{in}} \right)  \enspace.
\end{equation}

\comment{We can obtain approximate kinks in the lmd path by closed form solution and no need to invert multiple times.}

\paragraph{Path \wrt $z$ ($\lambda$ is fixed).}

We have the two situations for a change in the active set:

\begin{itemize}
\item A non-zero variable becomes zero \ie
\begin{equation} \label{eq:zfixed_z_to_nz}
    \exists j \in A(z, \lambda) \text{ such that }: \beta_j(z, \lambda) \neq 0 \text{ and } \beta_j(z_{j, \rm{out}}, \lambda) = 0 \enspace. 
\end{equation}
\item A zero variable becomes nonzero \ie

\begin{equation} \label{eq:zfixed_nz_to_z}
    \exists j \in A^c(z, \lambda) \text{ such that }: |X_{j}^{\top}\nabla f_{z_{j, \rm{in}}}(X\beta(z_{j, \rm{in}}, \lambda))| = \lambda \enspace.
\end{equation}
\end{itemize}

Then 
\begin{equation}\label{eq:next_z}
    z_{\rm{next}} = \max\left(\max_{j \in A(z, \lambda)} z_{j, \rm{out}},\; \max_{j \in A^c(z, \lambda)} z_{j, \rm{in}} \right) \enspace. 
\end{equation}

The core drawbacks is that \cref{eq:complex_primal} is non-linear in $z$ which makes the kink finder more complicated. Hence we need to use a root-finding (\eg bisection search) algorithm to estimate accurately the root. This require re-computing both 
$\bigg(X_{A}^{\top} \nabla^2 f_z(q_0)^\top X_A\bigg)^{-1}$ and $\nabla f_z(q_0)$ at every trial value $z$ \ie a dozen number of times which can be expensive.

\comment{To overcome this issue, we propose to linearize both in the first and second variable}

We remind that $\nabla f_z(q) = \partial_2 f(y(z), q)$; both notation will be used for simplicity or clarity.

Using a first order approximation, we have
\begin{equation}\label{eq:etash_linearization}
    \partial_2 f(y(z), q) \approx \partial_2 f(y(z_0), q_0) + 
    \partial_{2\, 1} f(y(z_0), q_0)^\top (y(z) - y(z_0)) +
    \partial_{2\, 2} f(y(z_0), q_0)^\top (q - q_0) 
\end{equation}

Using back the compact notation, we have $\partial_{2\, 2} f(y(z_0), q_0) = \nabla^2 f_{z_0}( q_0)$ and $\partial_2 f(y(z_0), q_0) = \nabla f_{z_0}(q_0)$.

Also, we have $y(z) - y(z_0) = (0, \ldots, 0, z - z_0)$, which implies that
$$ \partial_{2\, 1} f(y(z_0), q_0)^\top (y(z) - y(z_0)) = \tilde \partial_{n+1} f(z_0) (z - z_0) $$

where we denoted $\tilde \partial_{n+1} f(z_0)$ the last coordinate of $\partial_{2\, 1} f(y(z_0), q_0)$.

Finally, we can plug the linear approximation into the optimality condition and obtain
\begin{align}
    - \lambda v_A(z, \lambda)  &= X_{A}^{\top}\nabla f_z(q) \nonumber\\
     - \lambda v_A(z, \lambda) &\overset{\eqref{eq:etash_linearization}}{\approx}
    X_{A}^{\top} \bigg( \nabla f_{z_0}(q_0) + 
    \tilde \partial_{n+1} f(z_0) (z - z_0)  +
    \nabla^2 f_{z_0}( q_0)^\top (q - q_0) \bigg)  \nonumber\\
    X_{A}^{\top} \nabla^2 f_{z_0}(q_0)^\top q &\approx X_{A}^{\top} \bigg(\nabla^2 f_{z_0}(q_0)^\top q_0 - \nabla f_{z_0}(q_0) - \tilde \partial_{n+1} f(z_0) (z - z_0) \bigg) - \lambda v_A(z, \lambda) \nonumber
\end{align}

Hence
\begin{equation}\label{eq:etash_primal}
    \beta_A(z, \lambda) \approx  \bigg(X_{A}^{\top} \nabla^2 f_{z_0}(q_0)^\top X_A\bigg)^{-1} \bigg( X_{A}^{\top} \big(\nabla^2 f_{z_0}(q_0)^\top q_0 - \nabla f_{z_0}(q_0) - \tilde \partial_{n+1} f(z_0) (z - z_0) \big] - \lambda v_A(z, \lambda) \bigg) 
\end{equation}

Now the \Cref{eq:etash_primal} is linear both in $z$ and $\lambda$ but cheap to compute whereas \cref{eq:complex_primal} capture the non linearity in $z$.

The point of this section was to show that several more or less precise approximations can be easily constructed, and they lead to different properties. For example, for optimization purposes, it is more interesting, but unfortunately more costly, to capture the non-linearity of the solution path as much as possible. We haven't taken this option in this article, as we've observed in the examples we've tested that prediction accuracy is more important than estimation (of the optimal solution) accuracy when it comes to calculating conformal prediction sets. 

Another interesting approach could be to adopt paths based on checking the support of the optimal solution \citep{Ndiaye_Takeuchi22} when the input data of the $z$ or $\lambda$ of the problem changes. Among other things, this ensures that the active sets used always contain the optimum active set at all points.

\end{document}